%% file: main.tex
\documentclass{article} %
\usepackage{iclr_style/iclr2021_conference,times}

\input{iclr_style/math_commands.tex}

\usepackage{hyperref}
\usepackage{url}
\usepackage{iclr_style/style}
\usepackage{subcaption}
\usepackage{array}
\usepackage{booktabs}
\usepackage{makecell}
\usepackage{multirow}
\usepackage{float}
\usepackage[margin=10pt,font=small,labelfont=bf,
labelsep=endash]{caption}

\usepackage{chngcntr}

\title{For self-supervised learning,\\ Rationality implies generalization, provably}

\author{Yamini Bansal$^*$ \\
Harvard University
\And  Gal Kaplun\thanks{Equal contribution.  Email: \texttt{\{ybansal, galkaplun\}@g.harvard.edu}}\\
Harvard University
\And  Boaz Barak\thanks{Email: \texttt{b@boazbarak.org}.}
\\%[0.5ex]
Harvard University 
}

\usepackage{listings}

\DeclareFixedFont{\ttb}{T1}{txtt}{bx}{n}{8} %
\DeclareFixedFont{\ttm}{T1}{txtt}{m}{n}{8}  %

\usepackage{color}
\definecolor{deepblue}{rgb}{0,0,0.5}
\definecolor{deepred}{rgb}{0.6,0,0}
\definecolor{deepgreen}{rgb}{0,0.5,0}

\usepackage{thmtools}
\usepackage{thm-restate}

\usepackage{xspace}

\newcommand{\TestAcc}{\mathsf{Test}}
\newcommand{\TrainAcc}{\mathsf{Train}}
\newcommand{\NTrainAcc}{\mathsf{NTrain}}

\newcommand{\Dtrain}{\cD_{\text{\tiny train}}}
\newcommand{\Dtest}{\cD_{\text{\tiny test}}}

\newcommand{\Tpre}{T_{\text{\tiny pre}}}
\newcommand{\Tfit}{T_{\text{\tiny fit}}}

\usepackage{bbm}

\renewcommand{\vec}[1]{\boldsymbol{#1}}
\renewcommand{\vx}{\vec{x}}
\renewcommand{\vy}{\vec{y}}
\renewcommand{\vr}{\vec{r}}

\newcommand{\characteristic}{\mathbbm{1}}

\let\Pr\relax
\DeclareMathOperator*{\Pr}{\mathrm{\mathbf{Pr}}}

\definecolor{robustness}{HTML}{2A8621} %
\definecolor{rationality}{HTML}{A18803} %
\definecolor{memorization}{HTML}{e50000} %
\definecolor{generalization}{HTML}{0343df} %

\newcommand{\comp}{\mathsf{C}}

\newcommand{\cmdl}{\mathsf{C}^{\text{\tiny mdl}}}
\newcommand{\ccmi}{\mathsf{C}^{\text{\tiny pc}}}
\newcommand{\cmi}{\mathsf{C}^{\text{\tiny dc}}}

\iclrfinalcopy %
\begin{document}

\maketitle

\begin{abstract}
We prove a new upper bound on the generalization gap of classifiers that are obtained by first using self-supervision to learn a representation $r$ of the training~data, and then fitting a simple (e.g., linear) classifier $g$ to the labels.
Specifically, we show that (under the assumptions described below) the generalization gap of such classifiers tends to zero if $\comp(g) \ll n$, where $\comp(g)$ is an appropriately-defined measure of the simple classifier $g$'s complexity, and $n$ is the number of training samples. We stress that our bound is \emph{independent} of the complexity of the representation $r$. 

We do not make any structural or conditional-independence  assumptions on the representation-learning task, which can use \emph{the same training dataset} that is later used for classification.
Rather, we assume that the training procedure satisfies certain natural \emph{noise-robustness} (adding small amount of label noise causes small degradation in performance) and  \emph{rationality}  (getting the wrong label is not better than getting no label at all) conditions that widely hold across many standard architectures.
We show that our bound is non-vacuous for many popular representation-learning based classifiers on CIFAR-10 and ImageNet, including SimCLR, AMDIM and BigBiGAN.
\end{abstract}

\input{intro}

\makeatletter
\renewcommand{\thetheorem}{\Roman{theorem}}%
\makeatother

\input{threegaps}

\input{experiments}

\input{app-rationality}

\input{discussion}

\newpage
\appendix

\setcounter{table}{0}
\counterwithin{table}{section}
\renewcommand{\thetable}{\thesection.\arabic{table}}
\counterwithin{figure}{section}
\renewcommand{\thefigure}{\thesection.\arabic{figure}}

\input{appendix}

\end{document}

%% file: iclr_style/math_commands.tex
\usepackage{amsmath,amsfonts,bm}

\def\eqref#1{equation~\ref{#1}}

\def\1{\bm{1}}

\def\vr{{\bm{r}}}

\def\vx{{\bm{x}}}
\def\vy{{\bm{y}}}

\DeclareMathAlphabet{\mathsfit}{\encodingdefault}{\sfdefault}{m}{sl}
\SetMathAlphabet{\mathsfit}{bold}{\encodingdefault}{\sfdefault}{bx}{n}

\DeclareMathOperator*{\E}{\mathbb{E}}

\newcommand{\R}{\mathbb{R}}

\newcommand{\cR}{\mathcal{R}}

%% file: intro.tex
\section{Introduction}

The current standard approach for classification is ``end-to-end supervised learning'' where one fits a complex (e.g., a deep neural network) classifier to the given training set \citep{efficientnet, resnet}.
However, modern classifiers are heavily \emph{over-parameterized}, and as demonstrated by \cite{ZhangBHRV17}, can fit 100\% of their training set even when given random labels as inputs (in which case test performance is no better than chance).
Hence, the training performance of such methods is by itself no indication of their performance on new unseen test points.

In this work, we study a different class of supervised learning procedures that have recently attracted significant interest.
These classifiers are obtained by: \textbf{(i)} performing pre-training with a self-supervised task (i.e., without labels) to obtain a complex representation of the data points, and then  \textbf{(ii)} fitting a simple (e.g., linear) classifier on the representation and the labels.
Such \emph{``\textbf{S}elf-\textbf{S}upervised + \textbf{S}imple''} (SSS for short) algorithms are commonly used in natural language processing tasks \citep{bert,gpt3}, and have recently found uses in other domains as well~\citep{wav2vec,ssspeach,ssflow}.\footnote{In this work we focus only on algorithms that learn a representation, ``freeze'' it, and then perform classification using a simple classifier. We do not consider algorithms that ``fine tune'' the entire representation.}

Compared to standard ``end-to-end supervised learning'', SSS algorithms have several practical advantages.
In particular, SSS algorithms can incorporate additional unlabeled data, the representation obtained can be useful for multiple downstream tasks, and they can have improved out-of-distribution performance \citep{sss-hendryks}. Moreover, recent works show that even without additional unlabeled data, SSS algorithms can get close to state-of-art accuracy in several classification tasks \citep{simclrv2, moco,pirl,cmc}. For instance, SimCLRv2 \citep{simclrv2} achieves $79.8\%$ top-1 performance on ImageNet with a variant of ResNet-152, on par with the end-to-end supervised accuracy of this architecture at~$80.5\%$. 

We show that SSS algorithms have another advantage over standard supervised learning---they often have a small \emph{generalization gap}  between their train and test accuracy, and we \emph{prove non-vacuous bounds} on this gap. We stress that SSS algorithms use over-parameterized models to extract the representation, and reuse the \emph{same training data} to learn a simple classifier on this representation. Thus, the final classifier they produce has high complexity by most standard measures and the resulting  representation could ``memorize" the training set. Consequently, it is not a priori evident that their generalization gap will be small.

Our bound is obtained by first noting that the generalization gap of \emph{every} training algorithm is bounded by the sum of three quantities, which we name the \textbf{Robustness gap},  \textbf{Rationality gap}, and  \textbf{Memorization gap} (we call this the \textbf{RRM bound}, see Fact~\ref{RRMbound}).
We now describe these gaps at a high level, deferring the formal definitions to Section~\ref{sec:formal}.
All three gaps involve comparison with a setting where we inject \emph{label noise} by replacing a small fraction $\eta$ of the labels with random values.

The \emph{robustness gap} corresponds to the amount by which training performance degrades by noise injection.
That is, it equals the difference between the standard expected training accuracy (with no label noise) and the expected training accuracy in the noisy setting; in both cases, we measure accuracy with respect to the original (uncorrupted) labels.
The robustness gap is nearly always small, and sometimes provably so (see Section~\ref{sec:three-gaps}).

The \emph{rationality gap} corresponds to the difference between performance on the noisy training samples (on which the training algorithm gets the wrong label) and test samples (on which it doesn't get any label at all), again with respect to uncorrupted labels. An optimal Bayesian procedure would have zero rationality gap, and we show that this gap is typically zero or small in practice.

\begin{figure}[!b]
     \centering
         \includegraphics[width=0.9\textwidth]{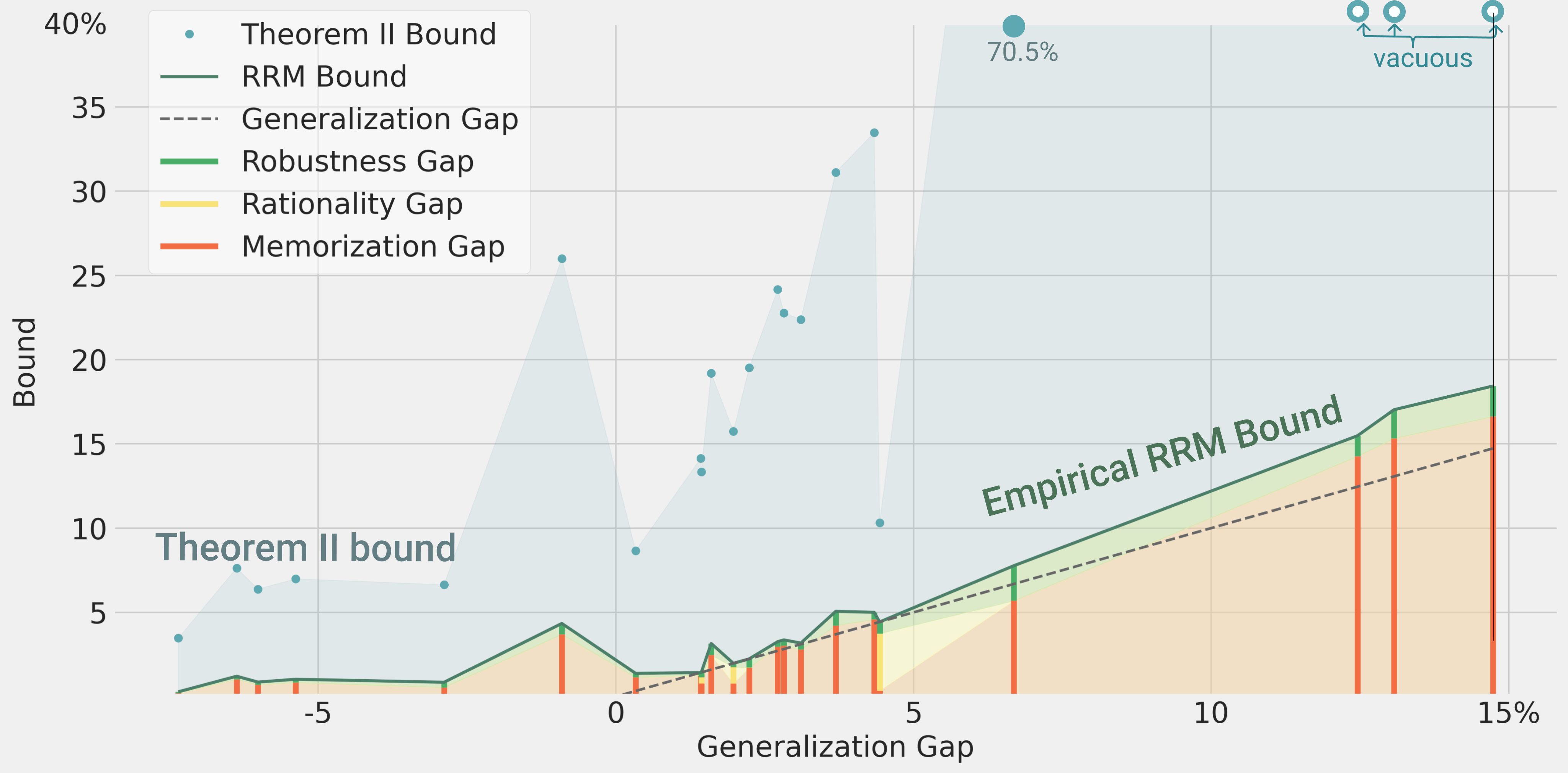}
        \caption{\textbf{Empirical RRM bound.} The components of the RRM bound, as well as the upper bound of Theorem~\ref{thm:main} for a variety of SSS models on the CIFAR-10 dataset with noise $\eta=0.05$. \\
        Each vertical line corresponds to a single model (architecture + self-supervised task + fitting algorithm) and plots the RRM bound for this model. The green component corresponds to robustness, yellow to rationality, and red to memorization. The $x$ axis is the generalization gap, and so the RRM bound is always above the dashed $x=y$ line. A negative generalization gap can occur in algorithms that use augmentation. The blue dots correspond to the bound on the generalization gap obtained by replacing the memorization gap with the bound of Theorem~\ref{thm:main}. See Sections~\ref{sec:experiments} and~\ref{app:fig-details} for more information.\\\\
        }
        \label{fig:intro}
\end{figure}

The \emph{memorization gap}, which often accounts for the lion's share of the generalization gap, corresponds to the difference in the noisy experiment between the training accuracy on the entire train set and the training accuracy on the samples that received the wrong label (both measured with respect to uncorrupted labels).
The memorization gap can be thought of as quantifying the extent to which the classifier can ``memorize'' noisy labels, or act differently on the noisy points compared to the overall train set.
The memorization gap is large in standard ``end-to-end supervised training''. In contrast, our main theoretical result is that for SSS algorithms, the memorization gap is small if the simple classifier has small complexity, \emph{independently} of the complexity of the representation.
As long as the simple classifier is under-parameterized (i.e., its complexity is asymptotically smaller than the sample size), our bound on the memorization gap tends to zero. When combined with small rationality and robustness, 
we get concrete non-vacuous generalization bounds for various SSS algorithms on the CIFAR-10 and ImageNet datasets (see Figures~\ref{fig:intro} and~\ref{fig:IN-MAIN}).

\paragraph{Our results.} In a nutshell, our contributions are the following:

\begin{enumerate}

    \item Our main theoretical result (Theorem~\ref{thm:main}) is that the \emph{memorization gap} of an SSS algorithm is bounded by $O(\sqrt{ C / n})$ where $C$ is the complexity of the simple classifier produced in the ``simple fit'' stage. This bound is oblivious to the complexity of the representation produced in the pre-training and does not make any assumptions on the relationship between the representation learning method and the supervised learning task.
    
    \item We complement this result with an empirical study of the robustness, rationality, and memorization gaps. We show that the RRM bound is typically non-vacuous, and in fact, often close to tight, for a variety of SSS algorithms on the CIFAR-10 and ImageNet datasets, including SimCLR (which achieves test errors close to its supervised counterparts). Moreover, in our experimental study, we demonstrate that the generalization gap for SSS algorithms is substantially smaller than their fully-supervised counterparts. See Figures~\ref{fig:intro}  and~\ref{fig:IN-MAIN} for sample results and Section~\ref{sec:experiments} for more details.
    
    \item We demonstrate that replacing the memorization gap with the upper bound of Theorem~\ref{thm:main} yields a \emph{non-vacuous generalization bound} for a variety of SSS algorithms on CIFAR-10 and ImageNet. Moreover, this bound gets tighter with more data augmentation.

    \item  The robustness gap is often negligible in practice, and sometimes provably so (see Section~\ref{sec:three-gaps}). We show that the rationality gap is small in practice as well. We also prove that a positive rationality gap corresponds to ``leaving performance on the table'', in the sense that we can transform a learning procedure with a large rationality gap into a procedure with better test performance (Theorem~\ref{thm:perfontable}).

\end{enumerate}

One way to interpret our results is that instead of obtaining generalization bounds under statistical assumptions on the distribution, we assume that the rationality and robustness gaps are at most some value (e.g., 5\%).
Readers might worry that we are ``assuming away the difficulty'', but small rationality and robustness gaps do \emph{not}  by themselves imply a small generalization gap. Indeed, these conditions widely hold across many natural algorithms (including not just SSS but also end-to-end supervised algorithms) with both small and large generalization gaps.
As discussed in Section~\ref{sec:three-gaps}, apart from the empirical evidence, there are also theoretical justifications for small robustness and rationality.
See Remark~\ref{rem:rationality} and Appendix~\ref{app:examples} for examples showing the necessity of these conditions.

\subsection{Related Work.}\label{sec:relworks} 

Our work analyses the generalization gap for supervised classifiers that first use self-supervision to learn a representation. We provide a brief exposition of the various types of self-supervised methods in Section~\ref{sec:experiments}, and a more detailed discussion in Appendix \ref{app:Tpre}. 

A variety of prior works have provided generalization bounds for supervised deep learning (e.g., \citet{ NIPS2017_7176,bartlett2017spectrally, dziugaite2017computing, Behnam-1805-12076, pmlr-v75-golowich18a, NIPS2019_9266}, and references therein).
However, many of these bounds provide vacuous guarantees for modern architectures (such as the ones considered in this paper) that have the capacity to memorize their entire training set \citep{ZhangBHRV17}.
While some non-vacuous bounds are known (e.g., \cite{ZhouVAAO19} gave a  96.5\% bound on the error of MobileNet on ImageNet), \cite{belkin2019reconciling, NagarajanK19} have highlighted some general barriers for bounding the generalization gaps of over-parameterized networks that are trained end-to-end.
For similar reasons, standard approaches such as Rademacher complexity cannot directly bound SSS algorithms' generalization gap (see Remark~\ref{rem:radamacher}).

Recently,  \cite{SaunshiPAKK19} and \cite{lee2020predicting} gave generalization bounds for self-supervised based classifiers.
The two works considered special cases of SSS algorithms, such as  \emph{contrastive learning} and \emph{pre-text tasks}.
Both works make strong statistical assumptions of (exact or approximate) \emph{conditional independence} relating the pre-training and classification tasks.
For example, if the pre-training task is obtained by splitting a given image $x$ into two pieces $(x_1,x_2)$ and predicting $x_2$ from $x_1$, then \cite{lee2020predicting}'s results require $x_1$ and $x_2$ to be approximately independent conditioned on their class $y$.
However, in many realistic cases, the two parts of the same image will share a significant amount of information not explained by the label.

Our work applies to general SSS algorithms without such statistical assumptions, at the expense of assuming bounds on the robustness and rationality gaps. There have been works providing rigorous bounds on the robustness gap or related quantities (See Section~\ref{sec:three-gaps}.). However, as far as we know, the rationality gap has not been explicitly defined or studied before. To bound the memorization gap, we use information-theoretic complexity measures. Various information-theoretic quantities have been proposed to bound generalization gap in previous work (see \citet{steinke} and references therein). While these works bounds generalization directly, we bound a different quantity---the memorization gap in the RRM decomposition.

\subsection{Paper Organization} 

Section~\ref{sec:formal}  contains formal definitions and statements of our results.
Section~\ref{sec:three-gaps} provides an overview of prior work and our new results on the three gaps of the RRM bound.
In Section~\ref{sec:experiments}, we describe our experimental setup and detail our empirical results.
Section~\ref{sec:disc} concludes the paper and discusses important open questions.
Section~\ref{sec:proof} contains the proof of Theorem~\ref{thm:main}, while Section~\ref{sec:perfontable} contains  the proof of  Theorem~\ref{thm:perfontable}. 
\cref{app:exp-methods} fully details our experimental setup.\footnote{We provide our code and data in:
\href{https://gitlab.com/harvard-machine-learning/rationality-generalization}{\color{blue} \underline{https://gitlab.com/harvard-machine-learning/}}.
}

\subsection*{Notation}

We use capital letters (e.g., $X$) for random variables, lower case letters (e.g., $x$) for a single value, and bold font (e.g., $\vec{x}$) for tuples (which will typically have dimension corresponding to the number of samples, denoted by $n$).
We use $x_i$ for the $i$-th element of the tuple $\vec{x}$.
We use calligraphic letters (e.g., $\cX,\cD$) for both sets and distributions.

%% file: threegaps.tex
\section{Formal statement of  results} \label{sec:formal}

A \emph{training procedure} is a (possibly randomized) algorithm $T$ that takes as input a train set $(\vx,\vy) = (x_i,y_i)_{i\in[n]} \in (\cX \times \cY)^n$ and outputs a classifier $f:\cX \rightarrow \cY$.
For our current discussion, we make no assumptions on the type of classifier output or the way that it is computed.
We denote the distribution over training sets in $(\cX\times\cY)^n$ by $\Dtrain$ and the distribution over test samples in $\cX\times\cY$ by $\Dtest$.\footnote{The train and test data often stem from the same distribution (i.e., $\Dtrain = \Dtest^n$), but not always (e.g., it does not hold if we use data augmentation). $\Dtest$ enters the RRM bound only via the rationality gap, so the assumption of small rationality may be affected if $\Dtrain \neq \Dtest^n$, but the RRM bound still holds.} 
The \emph{generalization gap} of a training algorithm $T$ with respect to a distribution pair $\cD = (\Dtrain,\Dtest)$ is the expected difference between its train accuracy (which we denote by $\TrainAcc_{\cD,T}$) and its test performance (which we denote by $\TestAcc_{\cD,T}$). We will often drop subscripts such as $\cD, T$ when they can be inferred from the context. We will also consider the \emph{$\eta$-noisy experiment}, which involves computing the classifier $\tilde{f} = T(\vx,\tilde{\vy})$ where $\tilde{y}_i = y_i$ with probability $1-\eta$ and is uniform  otherwise.

Our starting point is the following observation which we call the \textbf{RRM bound} (for \textbf{R}obustness, \textbf{R}ationality, and \textbf{M}emorization).
The quantities appearing in it are defined in Table~\ref{tab:acc}.

\begin{fact}[RRM bound] \label{RRMbound} 
For every noise parameter $\eta>0$, training procedure $T$ and distribution $\cD=(\Dtrain,\Dtest)$ over training sets and test samples, the RRM bound with respect to $T$ and  $\cD$ is, 
\medskip
\begin{equation*}
\underbrace{\TrainAcc - \TestAcc\vphantom{\Bigl[_+}}_{%
\let\scriptstyle\textstyle
\substack{\text{\textcolor{generalization}{Generalization}} \\  \text{\textcolor{generalization}{gap}}}} \leq
\underbrace{\Bigl[ \TrainAcc - \TrainAcc(\eta) \Bigr]_+}_{%
\let\scriptstyle\textstyle
\substack{\text{\textcolor{robustness}{Robustness}} \\ \text{\textcolor{robustness}{gap}}}}
+
\underbrace{\Bigl[ \NTrainAcc(\eta) - \TestAcc \Bigr]_+}_{%
\let\scriptstyle\textstyle
\substack{\text{\textcolor{rationality}{Rationality}} \\ \text{\textcolor{rationality}{gap}}}}
+
\underbrace{\Bigl[ \TrainAcc(\eta) - \NTrainAcc(\eta) \Bigr]_+}_{%
\let\scriptstyle\textstyle
\substack{\text{\textcolor{memorization}{Memorization}} \\  \text{\textcolor{memorization}{gap}}}}
\end{equation*}
\medskip
where we denote $x_+ = \max(x,0)$.
\end{fact}

\bgroup
\def\arraystretch{1.5}

\begin{table}[bht]
\centering
\caption{The measurements of accuracy in the RRM bound, all with respect to a training algorithm $T$, distributions $(\Dtrain,\Dtest)$ and parameter $\eta>0$. The \emph{robustness gap} is $\max(\TrainAcc-\TrainAcc(\eta),0)$, the \emph{rationality gap} is $\max(\NTrainAcc(\eta)-\TestAcc,0)$, and the \emph{memorization gap} is $\max( \TrainAcc(\eta)-\NTrainAcc(\eta),0)$.} 
    \label{tab:acc}
    \begin{tabular}{|c|m{1.8in}|m{2.5in}|}
 \hline
         Quantity& Training& \quad \quad \quad \quad \quad \quad
         Measurement
         \\ 
\hline
    $\TestAcc_{\cD,T}\;\;\;$ & $f=T(\vx,\vy)$ for $(\vx,\vy) \sim \Dtrain$ & $\Pr[f(x)=y]$ for $(x,y) \sim \Dtest$.   \\
    
    \hline
    
    $\TrainAcc_{\cD,T}\;\;\;$ & $f=T(\vx,\vy)$ for $(\vx,\vy) \sim \Dtrain$ & $\Pr[f(x_i)=y_i]$ for train sample $(x_i,y_i)$.  \\
 
 \hline
   $\TrainAcc_{\cD,T}(\eta) \; $ & $\tilde{f}=T(\vx,\tilde{\vy})$ for $(\vx,\vy) \sim \Dtrain$, $\tilde{y}_i = y_i$ w.p. $1-\eta$, uniform o/w & $\Pr[\tilde{f}(x_i)=y_i]$ for train sample $(x_i,\tilde{y}_i)$ where $y_i$ \emph{original} label for $x_i$.   \\
 \hline
 
 $\NTrainAcc_{\cD,T}(\eta)$ & $\tilde{f}=T(\vx,\tilde{\vy})$ for $(\vx,\vy) \sim \Dtrain$, $\tilde{y}_i = y_i$ w.p. $1-\eta$, uniform o/w & $\Pr[\tilde{f}(x_i)=y_i | \tilde{y}_i \neq y_i]$ for a corrupted train sample $x_i$ where $y_i$ \emph{original} label for $x_i$.    \\
 \hline
    \end{tabular}
\end{table}

\egroup

The RRM bound is but an observation, as it directly follows from the fact that $x_+ \geq x$ for every $x$. However, it is a very useful one.
As mentioned above, for natural algorithms, we expect both the \emph{robustness} and \emph{rationality} components of this gap to be small, and
hence the most significant component is the \emph{memorization gap}.
In this work we show a rigorous upper bound on this gap for SSS models.

We define formally an \emph{SSS Algorithm} to be a training procedure $T= (\Tpre,\Tfit)$ that is obtained by \textbf{(1)} first training $\Tpre$ on $\vx \in \cX^n$ to get a representation $r:\cX \rightarrow \cR$ and then \textbf{(2)} training $\Tfit$ on $(r(\vx),\vy)$ for $\vy\in \cY^n$ to obtain a classifier $g:\cR \rightarrow \cY$.
The classifier output by $T$ is $f:\cX \rightarrow\cY$ defined as $f(x)=g(r(x))$.  Our main theoretical result is the following.

\begin{theorem}[Memorization gap bound] \label{thm:main}
For every SSS Algorithm $T=(\Tpre,\Tfit)$, noise parameter $\eta>0$ and distribution $\cD$ over $\cX^n\times \cY^n$:
{
\setlength{\abovedisplayskip}{3pt}
\setlength{\belowdisplayskip}{3pt}
$$
\text{\textcolor{memorization}{Memorization gap}}(T) = \left( \TrainAcc_{T,\cD}(\eta) - \NTrainAcc_{T,\cD}(\eta) \right)_+ \leq O( \sqrt{\tfrac{\comp_\eta(\Tfit)}{n}} \cdot \tfrac{1}{\eta})
$$
}where $\comp_\eta(\Tfit)$ is a complexity measure of the second phase training procedure, which in particular is upper bounded by the number of bits required to describe the classifier $g$ (See Definition~\ref{def:complexity}.). 
\end{theorem}
\makeatletter
\renewcommand{\thetheorem}{\thesection.\arabic{theorem}}%
\@addtoreset{theorem}{section}%
\makeatother

\subsection{Complexity measures} \label{sec:theory}

We now define three complexity measures, all of which can be plugged in as the measure in Theorem~\ref{thm:main}.
The first one, $\cmdl$, is the minimum description length of a classifier.
The other two measures $\ccmi$ and $\cmi$ are superficially similar to Rademacher Complexity (cf. \cite{rademacher}) in the sense that they capture the ability of the hypothesis to correlate with random noise.

\begin{definition}[Complexity of training procedures] \label{def:complexity}
Let $T$ be a training procedure taking as input a set $( \vr,\vy)= \{(r_i,y_i) \}_{i=1}^n \in (\cR \times \cY)^n$ and outputting a classifier $g:\vr\to \cY$ and  let $\eta>0$.
For every training set $(\vr,\vy)$, we define the following three complexity measures with respect to $\vr,\vy,\eta$:

\begin{itemize}
    \vspace{-0.5ex}\item The \emph{minimum description length} of $T$ is defined as $\cmdl_{\vr,\vy,\eta}(T) \defeq H(g)$ where we consider the model $g$ as a random variable arising in the $\eta$-noisy experiment.\footnote{The name ``minimum description length'' is justified by the operational definition of entropy relating it to the minimum amortized length of a prefix-free encoding of a random variable.}
    
    \vspace{-0.5ex}\item The \emph{prediction complexity} of $T$ is defined as $\ccmi_{\vr,\vy,\eta}(T) \defeq \sum_{i=1}^n I( g(r_i) ;  \tilde{y}_i)$ where the $\tilde{y}_i$'s are the labels obtained in the $\eta$-noisy experiment.
    
    \vspace{-0.5ex}\item The (unconditional) \emph{deviation complexity} of $T$ is defined as $\cmi_{\vr,\vy,\eta}(T) \defeq n \cdot I(g(r_i)-y_i \; ; \; \tilde{y}_i - y_i)$ where the random variables above are taken over $i\sim [n]$ and subtraction is done modulo $|\cY|$, identifying $\cY$ with the set $\{0,\ldots, |\cY|-1\}$.

\end{itemize}
\end{definition}

Conditioned on $\vy$ and the choice of the index $i$, the deviations $g(r_i)-y_i$ and $\tilde{y}_i-y_i$ determine the predictions $g(r_i)$ and noisy labels $\tilde{y}_i$, and vice versa.
Hence we can think of $\cmi$ as an ``averaged'' variant of $\ccmi$, where we make the choice of the index $i$ part of the sample space for the random variables.
While we expect the two measures to be approximately close, the fact that $\cmi$ takes $i$ into the sample space makes it easier to estimate this quantity in practice without using a large number of experiment repetitions (see Figure~\ref{fig:cmi} for convergence rates).
The measure $\cmdl$ is harder to evaluate in practice, as it requires finding the optimal compression scheme for the classifier.
Section~\ref{sec:proof} contains the full proof of Theorem~\ref{thm:main}.
It is obtained by showing that: \textbf{(i)} for every $\vr,\vy,\eta$, and $T$ it holds that $\cmi_{\vr,\vy,\eta}(T) \leq \ccmi_{\vr,\vy,\eta}(T) \leq \cmdl_{\vr,\vy,\eta}(T)$, and \textbf{(ii)} for every SSS algorithm $T=(\Tpre,\Tfit)$ and distribution $\cD=(\Dtrain,\Dtest)$, the memorization gap of $T$ is at most 
\begin{equation}
\sqrt{\E_{(\vx,\vy)\sim \Dtrain} \cmi_{\Tpre(\vx),\vy,\eta}(\Tfit)} \;\bigg/\; \left(\eta \sqrt{2n}\right)  \;. \label{eq:empiricalcomplexity}
\end{equation}
It is the quantity (\ref{eq:empiricalcomplexity}) that we compute in our experiments.

\section{Proof of Theorem \ref{thm:main}} \label{sec:proof}

We now prove Theorem \ref{thm:main}. We start by relating our three complexity measures.
The following theorem shows that $\cmi$ is upper bounded by $\ccmi$, which in turn is bounded by the entropy of $g$.

\begin{theorem}[Relation of complexity measures] \label{thm:app-relations}
For every $\vr,\vy,\eta>0$, and $T$
\[
 \cmi_{\vr,\vy,\eta}(T) \leq \ccmi_{\vr,\vy,\eta}(T) \leq \cmdl(T)
\]
where $g$ is the classifier output by $T$ (considered as a random variable).
\end{theorem}

\begin{proof}
Fix $T,\vr,\vy,\eta$. We get $\tilde{\vec{y}}$ by choosing i.i.d random variables $N_1,\ldots,N_n$, each equalling $0$ with probability $1-\eta$ and uniform otherwise,
and letting $\tilde{y}_i = y_i + N_i \pmod{|\cY|}$.

We start by proving the second inequality $\ccmi_{\vr,\vy,\eta}(T) \leq H(g)$.
Let $g=T(\vr,\tilde{\vec{y}})$ and define $\vec{p}=(g(r_1),\ldots,g(r_n))$ be the vector of predictions.
Then,

\begin{equation}
\ccmi_{\vr,\vy,\eta}(T) = \sum_i I(p_i; \tilde{y}_i) = \sum_i I(p_i ; N_i)
\label{eq:compbound1}
\end{equation}

with the last equality holding since for fixed $y_i$, $N_i$ determines $\tilde{y}_i$ and vice versa.
However, since the full vector $\vec{p}$ contains only more information than $p_i$, the right-hand side of (\ref{eq:compbound1}) is at most  $\sum_{i=1}^n I(\vec{p};N_i) \leq I(\vec{p}\,;\,N_1,\ldots,N_n)$, using the fact that $N_i$ random variables are independent (see Lemma~\ref{lemma2}).
For a fixed $\vr$, the value of $\vec{p}$ is completely determined by $g$ and hence the entropy of $\vec{p}$ is at most $H(g)$, establishing the second inequality of the theorem.

We now turn to the first inequality $\cmi_{\vr,\vy,\eta}(T) \leq \ccmi_{\vr,\vy,\eta}(T)$.
Let $\Delta_i = p_i - y_i \pmod{|\cY|}$.
Then,

\begin{equation}
\tfrac{1}{n}\ccmi_{\vr,\vy,\eta}(T) = \E_{j\sim [n]} I(p_j;N_j) =  \E_{j\sim [n]} I(\Delta_j;N_j)
\label{eq:compbound2}
\end{equation}

since $p_i$ determines $\Delta_i$ and vice versa (given $y$).
But, since $N_j=N|i=j$  and $\Delta_j=\Delta|i=j$, the right-hand side  of (\ref{eq:compbound2}) equals

\begin{equation}
\E_{j \sim [n]} I(\Delta;N|i=j) = \E_{j\sim [n]} H(N|i=j) - H(N|\Delta,i=j)  \;.
\label{eq:compbound3}
\end{equation}
Since $N_1,\ldots,N_n$ are identically distributed, $H(N|i=j)=H(N)$ which means that the right-hand side of (\ref{eq:compbound3}) equals 
$$
H(N) - \E_{j \sim [n]} H(N|\Delta,i=j) \geq H(N) - H(N|\Delta) = I(\Delta;N) \,
$$
with the inequality holding since on average conditioning reduces entropy.
By definition $I(\Delta;N)=\tfrac{1}{n}\cmi_{\vr,\vy,\eta}(T)$, establishing what we wanted to prove.
\end{proof}

The complexity measures $\ccmi$ and $\cmi$ are defined with respect to a \emph{fixed} train set $(\vr,\vy)$, rendering them applicable for single training sets such as CIFAR-10 and ImageNet that arise in practice.
If $\cD$ is a distribution over $(\vr,\vy)$, then we define the complexity measures $\ccmi$ and $\cmi$ with respect to $\cD$ as the average of the corresponding measure with respect to $(\vr,\vy) \sim \cD$.
We now restate Theorem~\ref{thm:main}:

\begin{theorem}[Theorem~\ref{thm:main}, restated] \label{thm:app-mainrestate}
Let $T= (\Tpre,\Tfit)$ be a training procedure obtained by first training $\Tpre$ on $\vx\in \cX^n$ to obtain a representation $r:\cX \rightarrow \cR$ and then training $\Tfit$ on $(r(\vx), \vy))$ where $\vy\in \cY^n$ to obtain a classifier $g:\cR \rightarrow \cY$.
Then, for every noise parameter $\eta>0$ and distribution $\Dtrain$ over $(\cX,\cY)^n$,
$$
\text{\textcolor{memorization}{Memorization gap}}(T) = \left(\TrainAcc_{\Dtrain,T}(\eta) - \NTrainAcc_{\Dtrain,T}(\eta)\right)_+ \leq \sqrt{\tfrac{\cmi_{\cR,\eta}(\Tfit)}{2n}} \cdot \frac{1}{\eta} 
$$
where $\cR$ is the distribution over $(\cR \times \cY)^n$ induced by $\Tpre$ on $\Dtrain$.
\end{theorem}

Note that the bound on the right-hand side is expressed only in terms of the complexity of the second stage $\Tfit$ and is independent of the complexity of $\Tpre$. The crux of the proof is showing (close to) independence between the corrupted indices and prediction deviation of $g$ resulting from the noise.

\begin{proof}
Let $(\vr,\vy)$ be sampled by first drawing $(\vx,\vy) \sim \Dtrain$ over $(\cX\times\cY)^n$ then applying $\vr=r(\vx)$ where $r=\Tpre(\vx)$.
Consider the sample space of sampling $\tilde{\vy}$ according to the $\eta$-noisy distribution with respect to $Y$, computing $g=\Tfit(\vr,\tilde{\vy})$,  and sampling $i \sim [n]$.
We define the following two Bernoulli random variables over this sample space:

$$Z = \characteristic_{\Delta = 0} = \begin{cases}
1 & g(R_i)=y_i\\
0 & otherwise
\end{cases}; \quad \quad B = \characteristic_{N\ne0} = \begin{cases}
1 & \tilde{y}_i \ne y_i\\
0 & otherwise
\end{cases}.$$

For a given $\vr,\vy$, since $Z$ is determined by $\Delta$ and $B$ is determined by $N$,  $I(Z;B) \leq I(\Delta; N)= \cmi_{\vr,\vy,\eta}(\Tfit)/n$.
By Lemma~\ref{lem:bernoullirv}, for every  Bernoulli random variables $B,Z$,

$$\left| \E[Z] - \E[Z | B=1] \right|  \leq \sqrt{\tfrac{1}{2}I(Z;B)}/ \E[B]$$

And hence in our case (since $\E[B]=\eta$),

$$\E[Z] - \E[Z|B=1] \leq \sqrt{\tfrac{\cmi_{\vr,\vy,\eta}(\Tfit)}{2n}} \cdot \tfrac{1}{\eta}\;.$$

But $\E[Z]$ corresponds to the probability that $g(r)=y$ for $(r,y)$ in the train set, while $\E[Z|B=1]$ corresponds to this probability over the noisy samples.
Hence the memorization gap is bounded  by

$$
\E_{(\vr,\vy) \sim \cR} \left[\sqrt{\tfrac{\cmi_{\vr,\vy,\eta}(\Tfit)}{2n}} \cdot \tfrac{1}{\eta} \right] \leq \tfrac{1}{\eta} \sqrt{\E_{(\vr,\vy) \sim \cR} \left[ \tfrac{\cmi_{\vr,\vy,\eta}(\Tfit)}{2n} \right]}  = 
\sqrt{\tfrac{\cmi_{\cR,\eta}(\Tfit)}{2n}} \cdot \frac{1}{\eta} 
$$
using the Jensen inequality and the concavity of square root for the first inequality.
\end{proof}

\section{The three gaps} \label{sec:three-gaps}

We now briefly describe what is known and what we prove about the three components of the RRM bound. We provide some additional discussions in Appendix~\ref{app:examples}, including ``counter-examples'' of algorithms that exhibit large values for each one of these gaps.

 \begin{figure}[ht]
     \centering
     \begin{subfigure}[b]{0.32\textwidth}
         \centering
         \includegraphics[width=\textwidth]{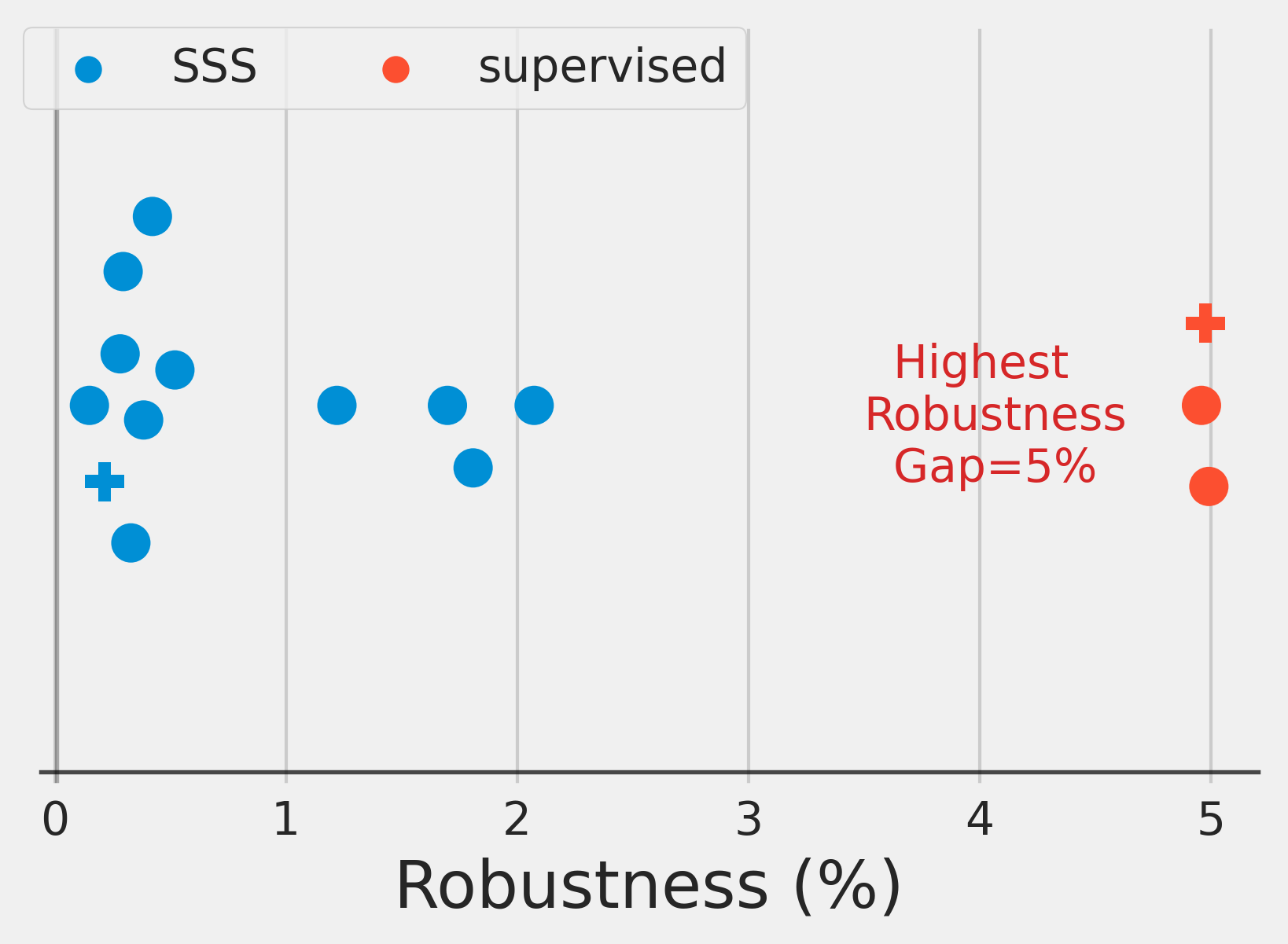}
     \end{subfigure}
     \begin{subfigure}[b]{0.32\textwidth}
         
         \includegraphics[width=\textwidth]{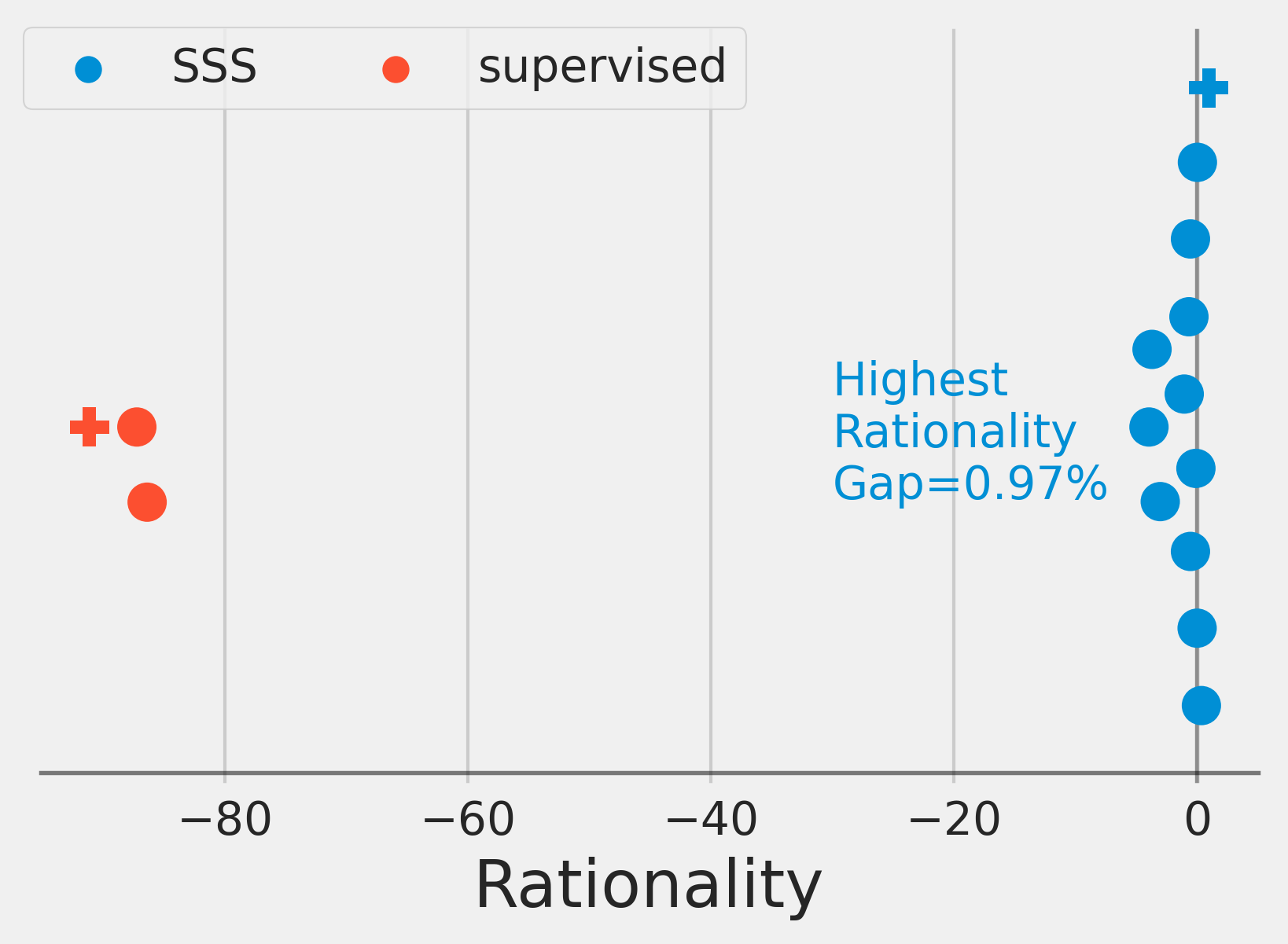}
     \end{subfigure}     
     \begin{subfigure}[b]{0.32\textwidth}
         
         \includegraphics[width=\textwidth]{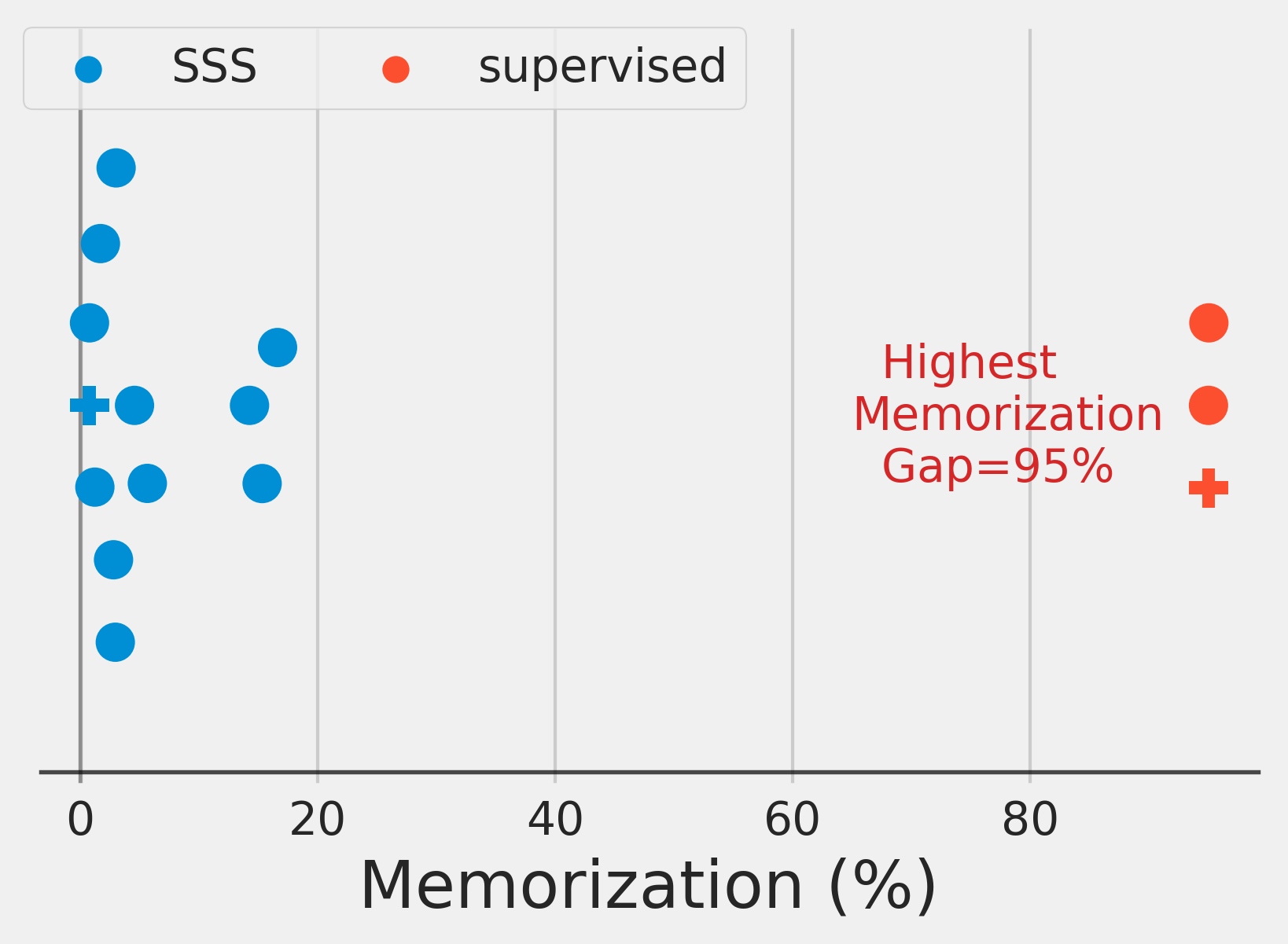}
     \end{subfigure}
        \caption{{\bf Robustness, Rationality, and Memorization for CIFAR-10.} Each blue point is a different combination of (architecture + self-supervised task + fitting algorithm). Each red point is a different architecture trained end-to-end with supervision. We use the `$+$' marker to denote the two best models of each type (SSS and supervised). No augmentations were added. Noise is $5\%$. 
        Details in \cref{app:fig-details}}
        \label{fig:threegaps}
\end{figure}

\paragraph{The robustness gap.} 
The robustness gap measures the decrease in training accuracy from adding $\eta$ noisy labels, measured with respect to the clean labels. The robustness gap and related notions such as   \emph{noise stability or tolerance} have been studied in various works (cf. \citet{frenay2013classification, manwani2013noise}). \emph{Interpolating classifiers} (with zero train error) satisfy $\TrainAcc(\eta) \geq 1-\eta$ and hence their robustness gap is at most  $\eta$ (see left panel of Figure~\ref{fig:threegaps}).
In SSS algorithms, since the representation is learned without using labels, the injection of label noise only affects the simple classifier, which is often \emph{linear}. Robustness guarantees for linear classifiers have been given previously by \citet{rudin2005stability}.
While proving robustness bounds is not the focus of this paper, we note in the appendix some simple bounds for least-squares minimization of linear classifiers and the (potentially inefficient) Empirical Risk Minimization algorithm (see  Appendices~\ref{sec:LSE} and~\ref{sec:ERMrobust}). 
Empirically, we observe that the robustness gap of SSS algorithms is often significantly smaller than $\eta$. (See left panels of Figure~\ref{fig:threegaps} and \cref{fig:exp-threegaps}.)

\paragraph{The rationality gap.} 
To build intuition for the rationality gap, consider the case where the inputs $x$ are images, and the label $y$ is either ``cat'' or ``dog''. A positive rationality gap means that giving the incorrect label ``dog'' for a cat image $x$ makes the output classifier \emph{more likely} to classify $x$ as a cat compared to the case where it is not given any label for $x$ at all.
Hence intuitively, a positive rationality gap corresponds to the training procedure being ``irrational'' or ``inconsistent''---wrong information should be only worse than no information, and we would expect the rationality gap to be zero or close to it.
Indeed, the rationality gap is always zero for \emph{interpolating classifiers} that fit the training data perfectly.
Moreover, empirically the rationality gap is often small for SSS algorithms, particularly for the better-performing ones.
(See middle panels of Figure~\ref{fig:threegaps} and \cref{fig:exp-threegaps}.)

We also show that positive  rationality gap corresponds to ``leaving performance on the table'' by proving the following theorem (see Section~\ref{sec:perfontable} for a formal statement and proof):

\begin{theorem}[Performance on the table theorem, informal] \label{thm:perfontable}
For every training procedure $T$ and distribution $\Dtest$, $\Dtrain=\Dtest^n$, there exists a training procedure $S$ satisfying $\TestAcc_S \geq \TestAcc_T + \text{\textcolor{rationality}{rationality gap}}(T) - o (1)$.
\end{theorem}

One interpretation of Theorem~\ref{thm:perfontable} is that we can always reduce the generalization gap to  $\text{\textcolor{robustness}{robustness}}+\text{\textcolor{memorization}{memorization}}$ if we are willing to move from the procedure $T$ to $S$. In essence, if the rationality gap is positive, we could include the test sample in the train set \emph{with a random label} to increase the test performance.
However, this transformation comes at a high computational cost; inference for the classifier produced by $S$ is as expensive as retraining from scratch.
Hence, we view Theorem~\ref{thm:perfontable} more as a ``proof of concept'' than as a practical approach for improving performance.

\begin{remark}[Why rationality?]\label{rem:rationality} Since SSS algorithms use a simple classifier (e.g., linear), the reader may wonder why we cannot directly prove bounds on the generalization gap.
The issue is that the representation used by SSS algorithms is still sufficiently over-parameterized to allow memorizing the training set samples. As a pedagogical example, consider a representation-learning procedure that maps a label-free training set $\vx$ to a representation $r:\cX \rightarrow\cR$ that has high quality, in the sense that the underlying classes become linearly separable in the representation space.
Moreover, suppose that the representation space has dimension much smaller than $n$, and hence a linear classifier would not be able to fit noise, meaning the resulting procedure will have a small memorization gap and small empirical Rademacher complexity.
Without access to the labels, we can transform $r$ to a representation $r'$ that on input $x$ will output $r(x)$ if $x$ is in the training set, and output the all-zero vector (or some other trivial value) otherwise.
Given sufficiently many parameters, the representation $r'$ (or a close-enough approximation) can be implemented by a neural network.
Since $r$ and $r'$ are identical on the training set, the procedure using $r'$ will have the same train accuracy, memorization gap, and empirical Rademacher complexity.
However, using  $r'$, one cannot achieve better than trivial accuracy on unseen test examples. This does not contradict the RRM bound since this algorithm will be highly irrational.
\end{remark}

\paragraph{The memorization gap.} 
The \emph{memorization gap} corresponds to the algorithm's ability to fit the noise (i.e., the gap increases with the number of fit noisy labels).
If, for example, the classifier output is \emph{interpolating}, i.e., it satisfies $f(x_i)=\tilde{y}_i$ for every $i$, then accuracy over the noisy samples will be $0$ (since for them $y_i \neq \tilde{y}_i$). In contrast, the overall accuracy will be in expectation at least $1-\eta$ which means that the memorization gap will be $\approx 1$ for small $\eta$.
However, we show empirically (see right panels of Figures~\ref{fig:threegaps}~and~\ref{fig:exp-threegaps}) that the memorization gap is small for many SSS algorithms and \emph{prove} a bound on it in Theorem~\ref{thm:main}.
When combined with small rationality and robustness, this bound results in non-vacuous generalization bounds for various real settings (e.g., 48\% for ResNet101 with SimCLRv2 on ImageNet, and as low as 4\% for MoCo V2 with ResNet-18 on CIFAR-10).
Moreover, unlike other generalization bounds, our bound decreases with \emph{data augmentation} (see \cref{fig:aug-cf10}).

\begin{remark}[Memorization vs. Rademacher] \label{rem:radamacher}
The memorization gap, as well the complexity measures defined in Section~\ref{sec:theory} have a superficial similarity to \emph{Rademacher complexity} \citep{rademacher}, in the sense that they quantify the ability of the output classifier to fit noise.
One difference is that Rademacher complexity is defined with respect to $100$\% noise, while we consider the $\eta$-noisy experiment for small $\eta$.
A more fundamental difference is that Rademacher complexity is defined via a supremum over all classifiers in some class. In contrast, our measures are defined with respect to a particular training algorithm.
As mentioned, \cite{ZhangBHRV17} showed that modern end-to-end supervised learning algorithm can fit 100\% of their label noise.
This is \emph{not the case} for SSS algorithms, which can only fit 15\%-25\% of the CIFAR-10 training set when the labels are completely random (see Table~\ref{table:zhang} in the appendix). 
However, by itself, the inability of an algorithm to fit random noise does not imply that the Rademacher complexity is small, and does not imply a small generalization gap.
Indeed, the example of Remark~\ref{rem:rationality} yields an SSS method with both small memorization gap and empirical Rademacher complexity, and yet has a large generalization gap.
\end{remark}

%% file: experiments.tex
\section{Empirical study of the RRM bound}
\label{sec:experiments}

In support of our theoretical results, we conduct an extensive empirical study of the three gaps and empirically evaluate the theoretical bound on the memorization gap (from Equation~(\ref{eq:empiricalcomplexity}) ) for a variety of SSS algorithms for the CIFAR-10 and ImageNet datasets. We provide a summary of our setup and findings below. For a full description of the algorithms and hyperparameters, see Appendix~\ref{app:exp-methods}.

 \begin{figure}[ht]
     \centering
     \includegraphics[width=\textwidth]{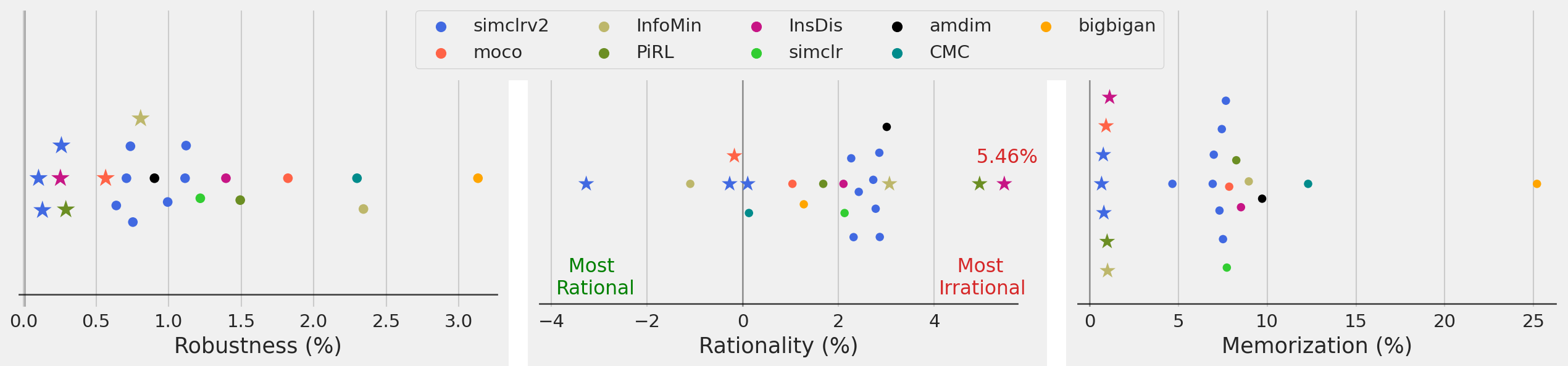}
     \caption{{\bf Robustness, Rationality and Memorization for ImageNet.} Each point represents a different combination of self-supervised learning algorithm (e.g., SimCLR), backbone architecture (e.g., ResNet-50) and simple classifier (e.g., linear classification). Star indicates experiments with 10 augmentations per training sample. Noise level is $\eta = 5\%$. Full experimental details in Section \ref{app:exp-methods}.}
        \label{fig:exp-threegaps}
\end{figure}

{\bf SSS Algorithms (}$\Tpre, \Tfit${\bf).} For the first phase of training $\Tpre$, we consider various self-supervised training algorithms that learn a representation \emph{without} explicit training labels. There are two main types of representation learning methods (1) \emph{Contrastive Learning}, which finds an embedding by pushing `'similar" samples closer, and (2) \emph{Pre-text tasks}, which hand craft a supervised task that is independent of downstream tasks, such as prediction the rotation angle of a given image \citep{rotation}. Our analysis is independent of the type of representation learning method, and we focus on methods that achieve high test accuracy when combined with the simple test phase. 
The list of methods included in our study is Instance Discrimination \citep{insdis}, MoCoV2 \citep{moco}, SimCLR \citep{simclrv1, simclrv2}, AMDIM \citep{amdim}, CMC \citep{cmc}, InfoMin \citep{infomin} as well as adversarial methods such as BigBiGAN \citep{bigbigan}. 

For the second phase of training (also known as the evaluation phase \citep{rep_eval_bench}), we consider simple models such as regularized linear regression, or small Multi-Layer Perceptrons (MLPs).
For each evaluation method, we run two experiments: 1) the clean experiment where we train $\Tfit$ on the data and labels $(\vx, \vy)$; 2) the $\eta$-noisy experiment where we train $\Tfit$ on $(\vx, \tilde\vy)$ where $\tilde\vy$ are the $\eta$ noised labels. Unless specified otherwise we set the noise to $\eta=5\%$.

{\bf Adding augmentations.} We investigate the effect of data augmentation on the three gaps and the theoretical bound. For each training point, we sample $t$ random augmentations ($t = 10$ unless stated otherwise) and add it to the train set. Note that in the noisy experiment two augmented samples of the same original point might be assigned with different labels. We use the same augmentation used in the corresponding self-supervised training phase. 

\begin{figure}[t]
\centering
\begin{minipage}{.5\textwidth}
\vspace{-0.3cm}
  \centering
  \includegraphics[width=\textwidth]{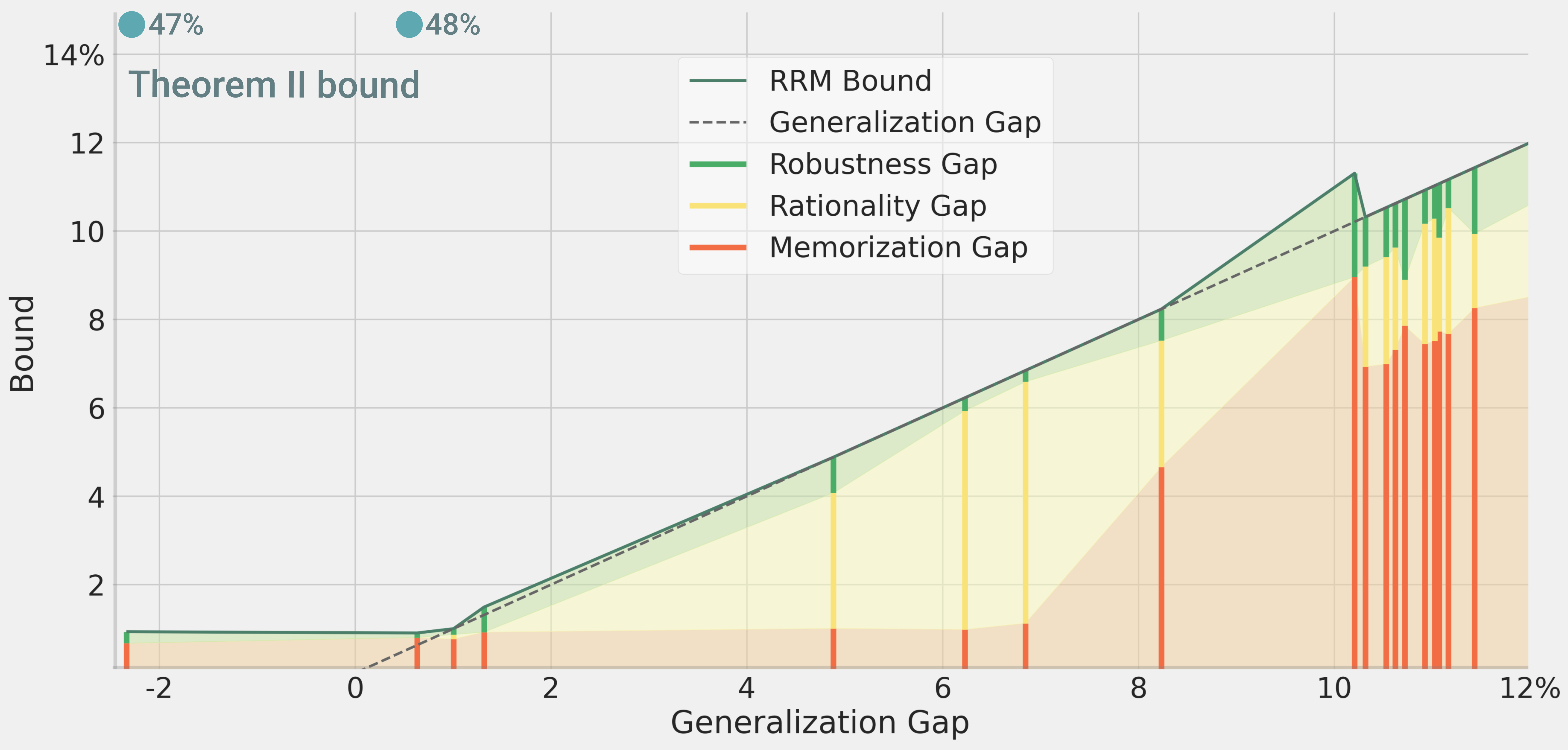}
  \caption{The RRM bound of SSS methods on ImageNet, with models sorted by the generalization gap. We plot the robustness, rationality and memorization gaps. Similar to  \cref{fig:intro}, for most models, the bound is tight and  is dominated by the memorization gap. Theorem~\ref{thm:main} bound is marked for the two leftmost models (we did not evaluate it for the others, for computational reasons). } \label{fig:IN-MAIN}
\end{minipage} %
\hspace{0.2cm}
\begin{minipage}{.45\textwidth}
  \centering
  \vspace{-0.6cm}
  \includegraphics[width=0.83\textwidth]{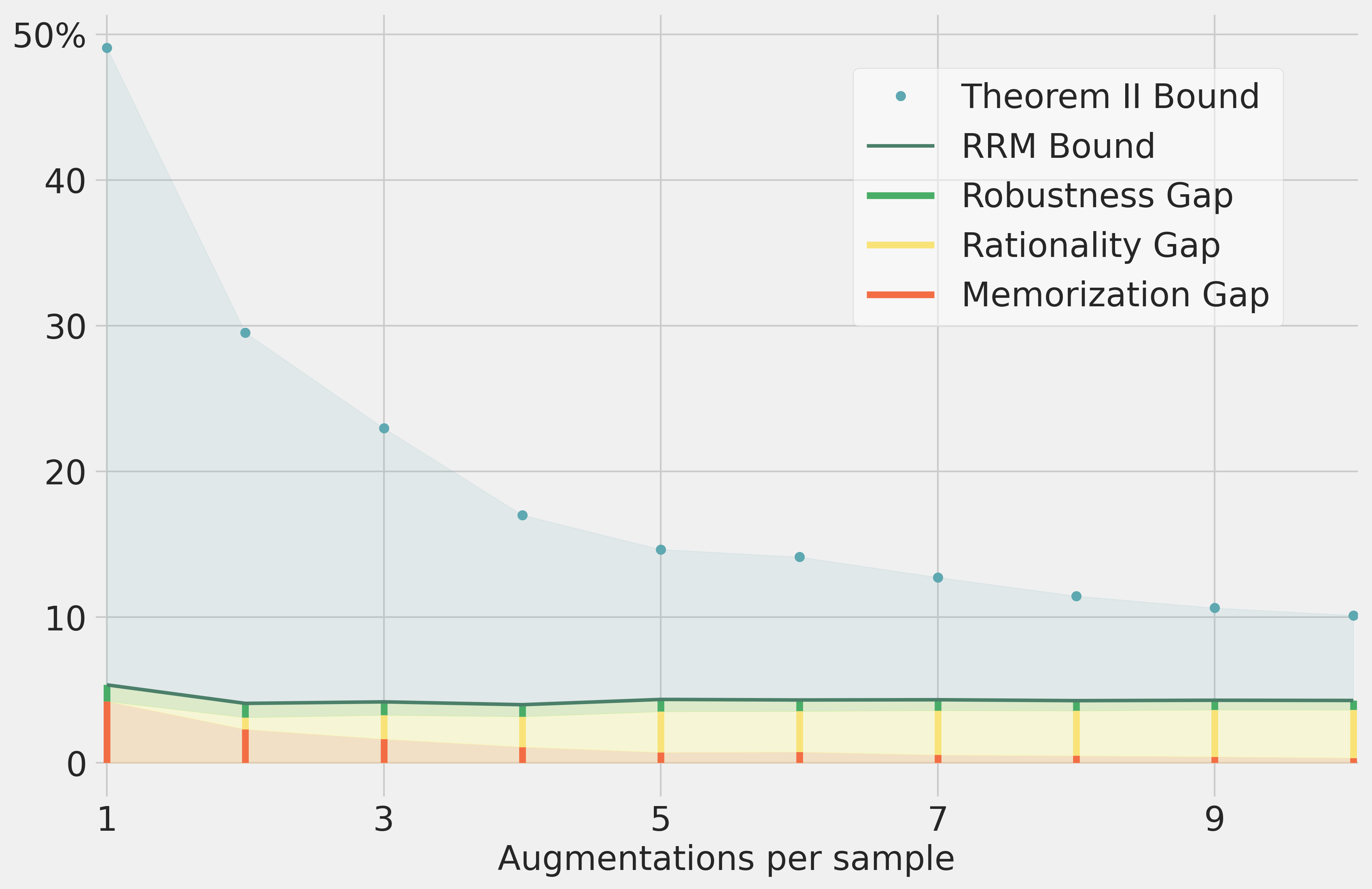}
 \caption{Empirical RRM for the AMDIM SSS model on CIFAR-10 with increasing  number of augmentations. While robustness and memorization gaps decrease, and so does our generalization bound, the rationality gap increases since $\cD_{\text{train}}$ and $\cD_{\text{test}}$ grow apart.}
 \label{fig:aug-cf10}
\end{minipage}
\end{figure}

{\bf Results.} Figures~\ref{fig:intro} and~\ref{fig:threegaps} provide a summary of our experimental results for CIFAR-10. The robustness and rationality gaps are close to zero for most SSS algorithms, while the memorization gap is usually the dominant term, especially so for models with larger generalization gap. Moreover, we see that $\cmi$ often produces a reasonably tight bound for the memorization gap, leading to a generalization bound that can be as low as $5$-$10\%$. In Figures~\ref{fig:exp-threegaps} and~\ref{fig:IN-MAIN} we give a summary of our experimental results for SSS algorithms on ImageNet. Again, the rationality and robustness gaps are bounded by small constants. Notice, that adding augmentations reduces memorization, but may lead to an increase in the rationality gap. This is also demonstrated in Figure~\ref{fig:aug-cf10} where we vary the number of data augmentations systematically for one SSS algorithm (AMDIM) on CIFAR-10. Since computing the Theorem II bound for ImageNet is computationally expensive we compute it only for two algorithms, which achieve non-vacuous bounds between ~$47$-$48\%$, with room for improvement (See Appendix \ref{app:convergence}.)

%% file: app-rationality.tex
\section{Positive rationality gap leaves room for improvement} \label{sec:perfontable}

We now prove the ``performance on the table theorem'' that states that we can always transform a training procedure with a positive rationality gap into a training procedure with better performance:

\begin{theorem}[Performance on the table theorem, restated] \label{thm:app-perfontable}
For every training procedure $T$ and $\Dtest,n,\eta$, if $\Dtrain =  \Dtest^n$  and $T$ has a positive rationality gap with respect to these parameters, then there exists a training procedure $S$ such that,
\begin{equation}
\TestAcc_{S,\cD} \geq \NTrainAcc_{T,\cD}(\eta) -o(1) =  \TestAcc_{T,\cD} + \text{\textcolor{rationality}{rationality-gap}}(T) - o(1) \label{eq:perfontable}
\end{equation} 
where $o(1)$ is a term that vanishes with $n$, and assuming that $\TrainAcc_{T,\cD}(\eta)  \geq \NTrainAcc_{T,\cD}(\eta)$.
\end{theorem}

The assumption, stated differently, implies that the memorization gap will be positive. We expect this assumption to be true for any reasonable training procedure $T$ (see right panel of \cref{fig:threegaps}), since performance on noisy train samples will not be better than the overall train accuracy. Indeed, it holds in all the experiments described in Section \ref{sec:experiments}.
In particular (since we can always add noise to our data), the above means that if the rationality gap is positive, we can use the above to improve the test performance of ``irrational'' networks. We now provide a proof for the theorem.

\begin{proof}
Let $T$ be a procedure with positive rationality gap that we are trying to transform. Our new algorithm $S$ would be the following:

\begin{itemize}
    \item \textbf{Training:} On input a training set $D=(\vx,\tilde{\vy}) \in (\cX\times\cY)^n$, algorithm $S$ does not perform any computation, but merely  stores the dataset $D$. Thus the ``representation'' of a point $x$ is simply $(x,D)$.
    
    \item \textbf{Inference:} On input a data point $x$ and the original training dataset $D$, algorithm $S$ chooses $i\sim [n]$ and lets $D'$ be the training set obtained by replacing $(x_i,y_i)$ with $(x,\tilde{y})$ where $\tilde{y}$ is chosen uniformly at random. We then compute $f=T(D')$, and output $f(x)$. 
\end{itemize}

First note that while the number of noisy samples could change by one by replacing $(x_i,y_i)$ with $(x,\tilde{y})$, since this number is distributed according to the Binomial distribution with mean $\eta n$ and standard deviation $\sqrt{(1-\eta)\eta n} \gg 1$,
this change can affect probabilities by at most $o(1)$ additive factor (since the statistical distance between the distribution $Binom(\eta,n)$ and $Binom(\eta,n)+1$ is $o(1)$).
If $\cY$ has $k$ classes, then with probability $1-1/k$ we will make $(x,\tilde{y})$ noisy ($y\ne \tilde{y}$) in which case the expected performance on it will be $\NTrainAcc_T(\eta)$. With probability $1/k$, we choose the correct label $y$ in which case performance on this sample will be equal to the expected performance on clean samples which by our assumptions is at least $\NTrainAcc_T(\eta)$ as well. Hence, the accuracy on the new test point is at least $\NTrainAcc_T(\eta)$.
\end{proof}

We stress that the procedure described above, while running in ``polynomial time'', is not particularly practical, since it makes \emph{inference} as computationally expensive as training.
However, it is a proof of concept that irrational networks are, to some extent, ``leaving performance on the table''.

%% file: discussion.tex
\vspace{-0.4cm}
\section{Conclusions and open questions}
\label{sec:disc}

This work demonstrates that SSS algorithms have small generalization gaps. 
While our focus is on the  \emph{memorization gap}, our work motivates more investigation of both the \emph{robustness} and \emph{rationality} gaps. 
In particular, we are not aware of any rigorous bounds for the rationality gap of SSS algorithms,  but we view our ``performance on the table'' theorem (Theorem~\ref{thm:perfontable}) as a strong indication that it is close to zero for natural algorithms.
Given our empirical studies, we believe the assumptions of small robustness and rationality conform well to practice.

Our numerical bounds are still far from tight, especially for ImageNet, where evaluating the bound (more so with augmentations) is computationally expensive.
Nevertheless, we find it striking that already in this initial work, we get non-vacuous (and sometimes quite good) bounds.
Furthermore, the fact that the empirical RRM bound is often close to the generalization gap, shows that there is significant room for improvement.

Overall, this work can be viewed as additional evidence for the advantages of SSS algorithms over end-to-end supervised learning.
Moreover, some (very preliminary) evidence shows that end-to-end supervised learning implicitly separates into a representation learning and classification phases \citep{morcos2018insights}. 
Understanding the extent that supervised learning algorithms implicitly perform SSS learning is an important research direction in its own right.
To the extent this holds, our work might shed light on such algorithms' generalization performance as well.

\section{Acknowledgements}

We thank Dimitris Kalimeris, Preetum Nakkiran, and Eran Malach for
comments on early drafts of this work.
This work supported in part by NSF award CCF 1565264, IIS 1409097, DARPA grant W911NF2010021, and a Simons Investigator Fellowship.  We also thank Oracle and Microsoft for grants used for computational resources. Y.B is partially supported by
MIT-IBM Watson AI Lab. Work partially performed while G.K. was an intern at Google Research.

%% file: appendix.tex
\section{Mutual information facts}\label{appendixa}
\begin{lemma} \label{lem:bernoullirv}
If $A,B$ are two Bernoulli random variables with nonzero expectation then,
$$| \E[A | B=1] - \E[A] |  \leq \sqrt{\tfrac{1}{2}I(A;B)}/ \E[B].$$
\end{lemma}
\begin{proof}
A standard relation between mutual information and KL-divergence gives,
$$
I(A;B) = D_{KL}(p_{A,B}||p_Ap_B).
$$
On the other hand, by the Pinsker inequality,
$$
\sup_{S \subseteq \{0,1\} \times \{0,1\}} |p_{A, B}(S) - p_{A\times B}(S)| \le \sqrt{\frac{1}{2}D_{KL}(p_{A,B}||p_Ap_B)}
 = \sqrt{\frac{1}{2}I(A,B)}.$$
Thus (letting $S=\{(1,1)\}$), 
$$\left| \Pr[A=1, B=1] - \Pr[A=1]  \Pr[B=1] \right| \leq  \sqrt{\tfrac{1}{2}I(A,B)}.$$
Consequently, 
$$\left| \E[A|B=1] - \E[A]\right| \leq \sqrt{\tfrac{1}{2}I(A,B)})/\E(B)$$
\end{proof}

\begin{lemma}
\label{lemma2}
For three random variables $W,X,Y$, s.t. $X$ and $Y$ are independent, 
$$I(W; X, Y) \ge I(W; X) + I(W; Y).$$
\end{lemma}
\begin{proof}

Using the chain rule for mutual information we have: 
$$I(W; X,Y) = I(W; X) + I(W; Y |X)$$
Since $X,Y$ are independent, $H(Y|X) = H(Y)$ and since conditioning only reduces entropy, we have $H(Y| W, X)\le H(Y|W)$. Combining the two we get,
\begin{align*}
    I(W;Y|X) &= H(Y|X) - H(Y | W, X)  \\
    &\ge H(Y) - H(Y|W) \\
    &= I(W; Y) 
\end{align*}
Thus we have that $I(W;X,Y) \ge I(W;X) + I(W;Y)$. 
\end{proof}
Note that by induction we can extend this argument to show that $I(W;X_1,...,X_n)\ge \sum I(W; X_i)$ where $X_i$ are mutually independent.

\input{app-experiments}

\section{Examples of algorithms with large gaps} \label{app:examples}

While we argued that SSS algorithms will tend to have small robustness, rationality, and memorization gaps, this does not hold in the worst case and there are examples of such algorithms that exhibit large gaps in each of those cases.

\subsection{Large robustness gap} \label{sec:lexamples:robustness}

Large robustness gap can only arise via computational (as opposed to statistical) considerations.
That is, if a training procedure outputs a classifier $f\in \cF$ that achieves on average accuracy $\alpha$ on a clean train set $(X,Y)$, then with high probability, if $(X,\tilde Y)$  is an $\eta$-noisy train set then \emph{there exists} $f \in \cF$ that achieves $\alpha(1-\eta)$ accuracy on this train set (by fitting only the ``clean'' points).

However, the training algorithm might not always be able to find such a classifier.
For example, if the distribution has the form $(x,y)=(x, \sum a_j x_j \mod 2)$ where $x\sim GF(2)^\ell=\mathbb Z_2^\ell$  and $a\in GF(2)^\ell$ is some hidden vector, then there is an efficient algorithm (namely Gaussian elimination) to find $a$ given the samples $(x,y)$ and hence get accuracy $1$.
However, for every $\epsilon>0$ and $\eta>0$, there is no known efficient algorithm that, given a $1-\eta$ perturbed equations of the form $\{ \langle a,x_i \rangle = \tilde{y}_i \}_{i\in [n]}$ finds $a' \in  GF(2)^\ell$ such that $\sum a'_j x_j = \sum a_j x_j \mod 2$ on a $1/2+\epsilon$ fraction of the $x$'s.
This is known as the \emph{learning parity with noise (LPN)} problem~\citep{BlumFKL93}.

The assumption of robustness is \emph{necessary} for a small generalization gap, in the sense that we can come up with (contrived) examples of algorithms that have small rationality and memorization gaps while still having large generalization gap.
For example, consider an algorithm $T$ that has large generalization gap (high train accuracy and small test accuracy), and suppose we augment to the following algorithm

$$T'(\vx,\vy) = \begin{cases}
                T(\vx,\vy)  &  \text{ if $\vy$ is ``clean''} \\
                0 &  \text{ if $\vy$ is ``noisy''} 
                \end{cases}$$
                
where $0$ denotes the constant zero function (e.g., some trivial classifier) and we use some algorithm to estimate whether or not the labels are noisy.
(Such estimates can often be achieved in many natural cases.)
The algorithm $T'$ will inherit the generalization gap of $T$, since that depends only on the experiment without noise. Since performance on noisy and clean training samples will be the same (close to random), $T'$ will have zero memorization gap. Since we have assumed small test accuracy, it will have zero rationality gap also.

\subsection{Large rationality gap}

As discussed in Section~\ref{sec:perfontable}, in the case that $\Dtrain = \Dtest^n$,  a robust algorithm with large rationality gap leaves ``performance on the table''.
We can obtain such algorithms by artificially dropping performance on the test data.
For example, in the SSS framework, since the representation $r$ is over-parameterized and can memorize the entire train set, we can consider the trivial representation 

$$r(x) = \begin{cases}x & \text{$x$ in train set} \\ 0 & \text{otherwise} \end{cases}$$

If we now train some simple classifier on $r(x)$ then it can have non-trivial performance on the noisy train samples, while getting trivial accuracy on all samples outside the train set.

In cases where $\Dtrain$ and $\Dtest$ are different (for example when $\Dtrain$ is an augmented version of $\Dtest$) then we can no longer claim that a large rationality gap corresponds to ``leaving performance on the table''. For example, we do observe (mild) growth in the rationality gap as we add more augmented points to the training set. 

\subsection{Large memorization gap}

It is not hard to find examples of networks with large memorization gap. Indeed, as mentioned before, any standard interpolating supervised learning algorithm will get a memorization gap close to $1$.

\input{app-robustness}
\newpage
\section{Large Tables}
\input{app-results-table}

\input{hyperparam-table}

%% file: app-experiments.tex
\section{Experimental details} \label{app:exp-methods}

We perform an empirical study of the RRM bound for a wide variety of self-supervised training methods on the ImageNet \citep{imagenet} and CIFAR-10 \citep{cifar10} training datasets. We provide a brief description of all the self-supervised training methods that appear in our results below. For each method, we use the official pre-trained models on ImageNet wherever available. Since very few methods provide pre-trained models for CIFAR-10, we train models from scratch. The architectures and other training hyper-parameters are summarized in \cref{tab:hyperparam-table-imagenet} and \cref{tab:hyperparam-table-cifar}. Since our primary aim is to study the RRM bound, we do not optimize for reaching the state-of-the-art performance in our re-implementations. For the second phase of training, we use L2-regularized linear regression, or small non-interpolating Multi-layer perceptrons (MLPs). 

\subsection{Self-supervised training methods (\texorpdfstring{$\Tpre$}{ Tpre})}
\label{app:Tpre}

There is a variety of self-supervised training methods for learning representations without explicit labels. The two main branches of self-supervised learning methods are:

\begin{enumerate} 
\item {\it Contrastive learning:} These methods seek to find an embedding of the dataset that pushes a \emph{positive} pair of images close together and a pair of \emph{negative} images far from each other. For example, two different augmented versions of the same image may be considered a positive pair, while two different images may be considered a negative pair. Different methods such as Instance Discrimination, MoCo, SimCLR, AMDIM, differ in the the way they select the positive/negative pairs, as well other details like the use of a memory bank or the encoder architecture. (See \cite{yadim} for detailed comparison of these methods.) 

\item {\it Handcrafted pretext tasks:} These methods learn a representation by designing a fairly general supervised task, and utilizing the penultimate or other intermediate layers of this network as the representation. Pretext tasks include a diverse range of methods such as predicting the rotation angle of an input image \citep{rotation}, solving jigsaw puzzles \citep{jigsaw_noroozi}, colorization \citep{colorization}, denoising images \citep{denoisingAE} or image inpainting \citep{contextAE}. 

\end{enumerate}

Additionally, adversarial image generation can be used for by augmenting a the image generator with an encoder \citep{bigbigan}. We focus primarily on contrastive learning methods since they achieve state-of-the-art performance. We now describe these methods briefly.

{\bf Instance Discrimination:} \citep{insdis} In essence, Instance Discrimination performs supervised learning with \emph{each} training sample as a separate class. They minimize the non-parametric softmax loss given below for each training sample $v = f_{\theta}(x)$ 

\begin{equation}
    J(\theta) = - \sum_{i=1}^n log \Bigg( \frac{\exp(v_i^Tv/\tau)}{ \sum_{j=1}^n \exp(v_i^Tv/\tau)} \Bigg)
\end{equation}

where $v_i = f_{\theta}(x_i)$ is the feature vector for the $i$-th example and $\tau$ is a temperature hyperparameter. They use memory banks and a contrastive loss (also known as Noise Contrastive Estimation or NCE \citep{nce}) for computing this loss efficiently for large datasets. So in this case, a positive pair is an image and itself, while a negative pair is two different training images.

{\bf Momentum Contrastive (MoCo):} \citep{moco} MoCo replaces the memory bank in Instance Discrimination with a momentum-based query encoder. MoCoV2 \citep{mocov2} applies various modifications over SimCLR, like a projection head, and combines it with the MoCo framework for improved performance.

{\bf AMDIM:} \citep{amdim} AMDIM uses two augmented versions of the same image as possitive pairs. For these augmentations, they use random resized crops, random jitters in color space, random horizontal flips and random conversions to grayscale. They apply the NCE loss across multiple scales, by using features from multiple layers. They use a modified ResNet by changing the receptive fields to decrease overlap between positive pairs.

{\bf CMC:} \citep{cmc} CMC creates two views for contrastive learning by converting each image into the Lab color space. L and ab channels from the same image are considered to be a positive pair, while those from two different images are considered to be a negative pair.

{\bf PiRL:} \citep{pirl} PiRL first creates a jigsaw transformation of an image (it divides an image into 9 patches and shuffles these patches). It treats an image and its jigsaw as a positive pair, and that of a different image as a negative pair. 

{\bf SimCLRv1 and SimCLRv2:} \citep{simclrv1, simclrv2} SimCLR also use strong augmentations to create positive and negative pairs. They use random resized crops, random Gaussian blurring and random jitters in color space. Crucially, they use a projection head that maps the representations to a 128-dimensional space where they apply the contrastive loss. They do not use a memory bank, but use a large batch size. 

{\bf InfoMin:} \citep{infomin} InfoMin uses random resized crops, random color jitters and random Gaussian blurring, as well as jigsaw shuffling from PiRL.

\subsection{Simple Classifier (\texorpdfstring{$\Tfit$}{ Tfit})}
After training the representation learning method, we extract representations $r$ for the training and test images. We do not add random augmentations to the training images (unless stated otherwise). Then, we train a simple classifier on the dataset $\{r(x_i), y_i\}_{i=1}^n $. We use a linear classifier in most cases, but we also try a small multi-layer perceptron (as long as it has few parameters and does not interpolate the training data). We add weight decay in some methods to achieve good test accuracy (see  \cref{tab:hyperparam-table-imagenet} and \cref{tab:hyperparam-table-cifar} for values for each method). For the noisy experiment, we set the noise level to $\eta = 5\%$. To compute the complexity bound $\cmi$ we run 20 trials (same experiment with different random seed) of the noisy experiment for CIFAR-10 and 50 trials for ImageNet.

\subsection{Experimental details for each plot}
\label{app:fig-details}
{\bf Figure \ref{fig:intro}.} This figure shows the robustness, rationality and memorization gap for various SSS algorithms trained on CIFAR-10. The type of self-supervised method, the encoder architecture, as well as the training hyperparameters are described in Table \ref{tab:hyperparam-table-cifar}. For the second phase $\Tfit$, we use L2-regularized linear regression for all the methods. For each algorithm listed in Table \ref{tab:hyperparam-table-cifar}, the figure contains 2 points, one without augmentations, and one with augmentations. Further, we compute the complexity measure $\cmi$ for all the methods. All the values (along with the test accuracy) are listed in Table \ref{tab:all-cifar}. 

{\bf Figure \ref{fig:threegaps}.} This figure shows the robustness, rationality and memorization for CIFAR-10 for all the same methods as in Figure \ref{fig:intro}. We only include the points without augmentation to show how rationality behaves when $(\Dtrain,\Dtest)$ are identical. All the values (along with the test accuracy) are listed in Table \ref{tab:all-cifar}. In addition, we add three end-to-end fully supervised methods (red circles) to compare and contrast the behavior of each of the gaps for SSS and supervised methods. For the supervised architectures, we train a Myrtle-5 \citep{mcnn} convolutional network, a ResNet-18 \citep{resnet} and a WideResNet-28-10 \citep{wrn} with standard hyperparameters.

{\bf \cref{fig:exp-threegaps} and \cref{fig:IN-MAIN}.} These figures show the robustness, rationality and memorization for the ImageNet dataset. The type of self-supervised method, the encoder architecture, as well as the training hyperparameters are described in Table \ref{tab:hyperparam-table-imagenet}. For the second phase $\Tfit$, we use L2-regularized linear regression for all the methods. The figures also contain some points with 10 augmentations per training image. Further, we compute the complexity measure $\cmi$ for all three methods---SimCLRv2 with architectures ResNet-50-1x and ResNet-101-2x. All the values (along with the test accuracy) are listed in Table \ref{tab:all-imagenet}. 

{\bf Figure \ref{fig:aug-cf10}} This figure shows the effect of increasing augmentations. We add $t =\{2,...,10 \}$ augmentations and re-train the simple classifier. We do this for the CIFAR-10 dataset, AMDIM self-supervised training with the AMDIM encoder and linear regression (see Table \ref{tab:hyperparam-table-cifar} for the hyperparameters).

\subsection{Additional Results}

\subsubsection{Generalization error of SSS algorithms} \label{sec:emp-radhemacher}

To show that SSS algorithms have qualitatively different generalization behavior compared to standard end-to-end supervised methods, we repeat the experiment from \cite{ZhangBHRV17}. We randomize all the training labels in the CIFAR-10 dataset and train 3 high-performing SSS methods on these noisy labels. For results see Table~\ref{table:zhang}. Unlike fully supervised methods, SSS algorithms do not achieve 100\% training accuracy on the dataset with noisy labels. In fact, their training accuracies are fairly low ($\approx15$-$25\%$). This suggests that the empirical Rademacher complexity is bounded. The algorithms were trained without any augmentations during the simple fitting phase for both SSS and supervised algorithms. The SSS methods were trained using parameters described in Table \ref{tab:hyperparam-table-cifar}.

\begin{table}[H]
\centering
\caption{Train and Test performance on 100\% label noise for fully supervised vs. SSS algorithms on CIFAR-10. The first row is from \citet{ZhangBHRV17}, while the second one is our results for SSS methods averaged over 5 runs without augmentations.}
\label{table:zhang}
\begin{tabular}{@{}cccccccc@{}} 
    \toprule
    \makecell{Training method}  & \makecell{Architecture/Method} & \makecell{Train Acc} & \makecell{Test Acc} \\ 
    \midrule

    \multirow{2}{*}{Supervised \citep{ZhangBHRV17}}  & Inception (no aug) & 100\% & 86\% \\ 
    
                        & (fitting random labels) & {\color{red}\bf 100\%} & {\color{red}\bf 10\%} \\ 
    \midrule
    
    \multirow{4}{*}{SSS}  &  SimCLR (ResNet-50) + Linear & 94\%  & 92\% \\ 
    
                        & (fitting random labels) & {\color{blue}\bf 22\%}& {\color{red}\bf 10\%} \\ 
                        &  AMDIM (AMDIM Encoder) + Linear & 94\%  & 87.4\% \\ 
    
                        & (fitting random labels) & {\color{blue}\bf 18\%} & {\color{red}\bf 10\%} \\ 
                        &  MoCoV2 (ResNet-18) + Linear & 69\%  & 67.6\% \\ 
    
                        & (fitting random labels) & {\color{blue}\bf 15\%} & {\color{red}\bf 10\%} \\

    \bottomrule
\end{tabular} \\
\end{table}

\subsection{RRM bound with varying noise parameter}
We now investigate the effect of varying noise levels on the three gaps as well as on the complexity. We see that the robustness gap increases as we add more noise---this is expected as noise should affect the clean training accuracy. We also observe that the memorization gap decreases, suggesting that $\cmi_\eta$ as a function of $\eta$ goes down faster than $\eta^2$ (see \cref{sec:theory}). The Theorem \ref{thm:main} bound on memorization gap also decays strongly with the $\eta$, becoming more tight as the noise increases. 

\begin{figure}[ht]
    \centering
    \includegraphics[width=0.75\textwidth]{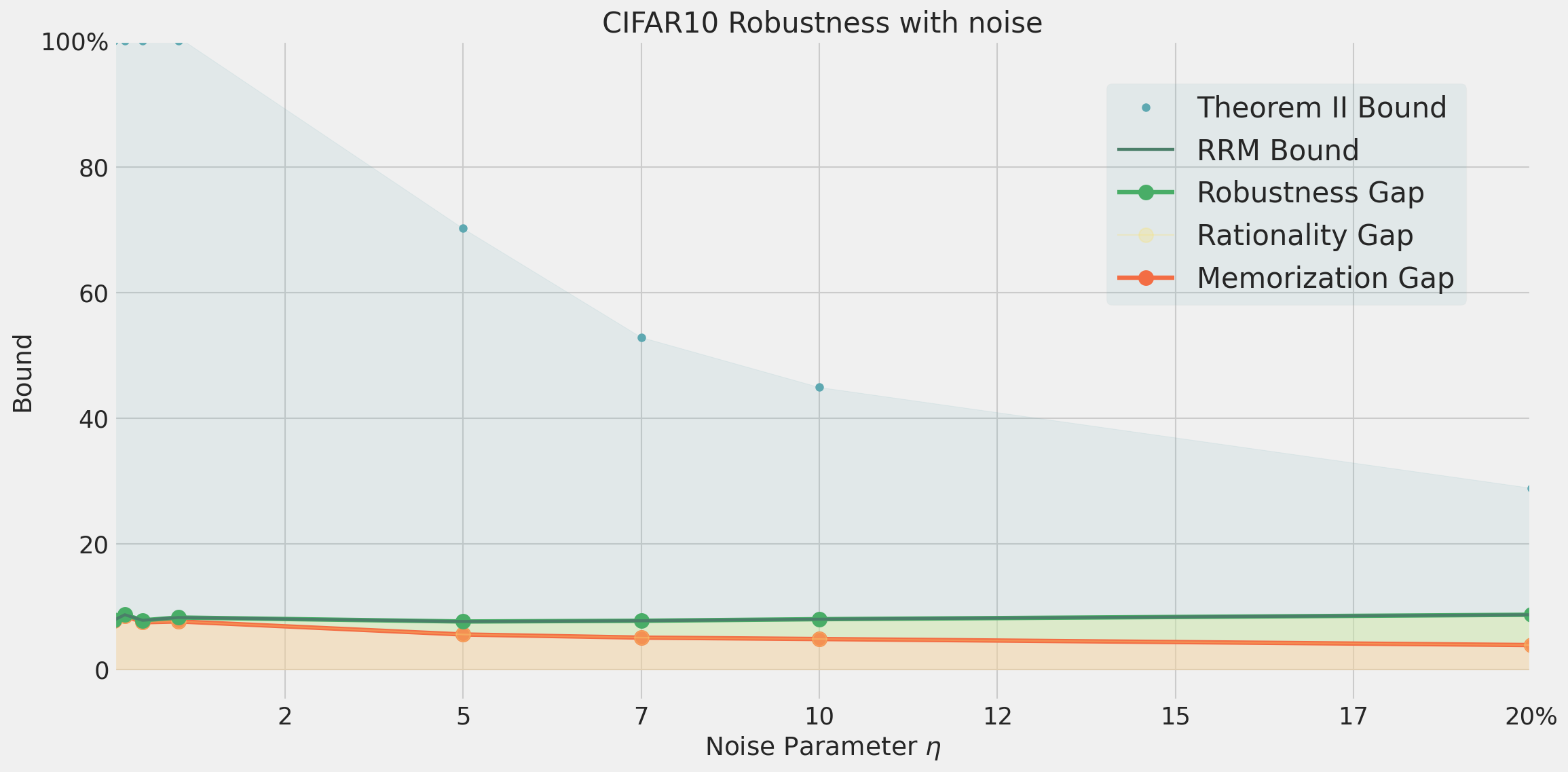}        
        \caption{RRM + bound with changing $\eta$}
        \label{fig:varyeta}
\end{figure}

\subsubsection{Convergence of complexity measures}
\label{app:convergence}
We now plot (see \cref{fig:cmi}) the complexity measures $\cmi$ and $\ccmi$ with increasing number of trials for one of the SSS algorithms. As expected, $\cmi < \ccmi$ and $\cmi$ converges in about 20 trials for CIFAR-10. On the other hand, the complexity computations for ImageNet need many more trials for convergence, since it contains about $10$ augmentations $\times 1.2$ million training samples making it cost prohibitive to compute for all the methods. For the CIFAR-10, we use AMDIM with the AMDIM encoder architecture without augmentations. For ImageNet, we use SimCLRv2 with the ResNet-101 architecture with 10 augmentations per training sample.

\begin{figure}[ht]

    \centering
    \begin{subfigure}{0.45\textwidth}
    \vspace{-0.1cm}
    \centering
    \includegraphics[width=\textwidth]{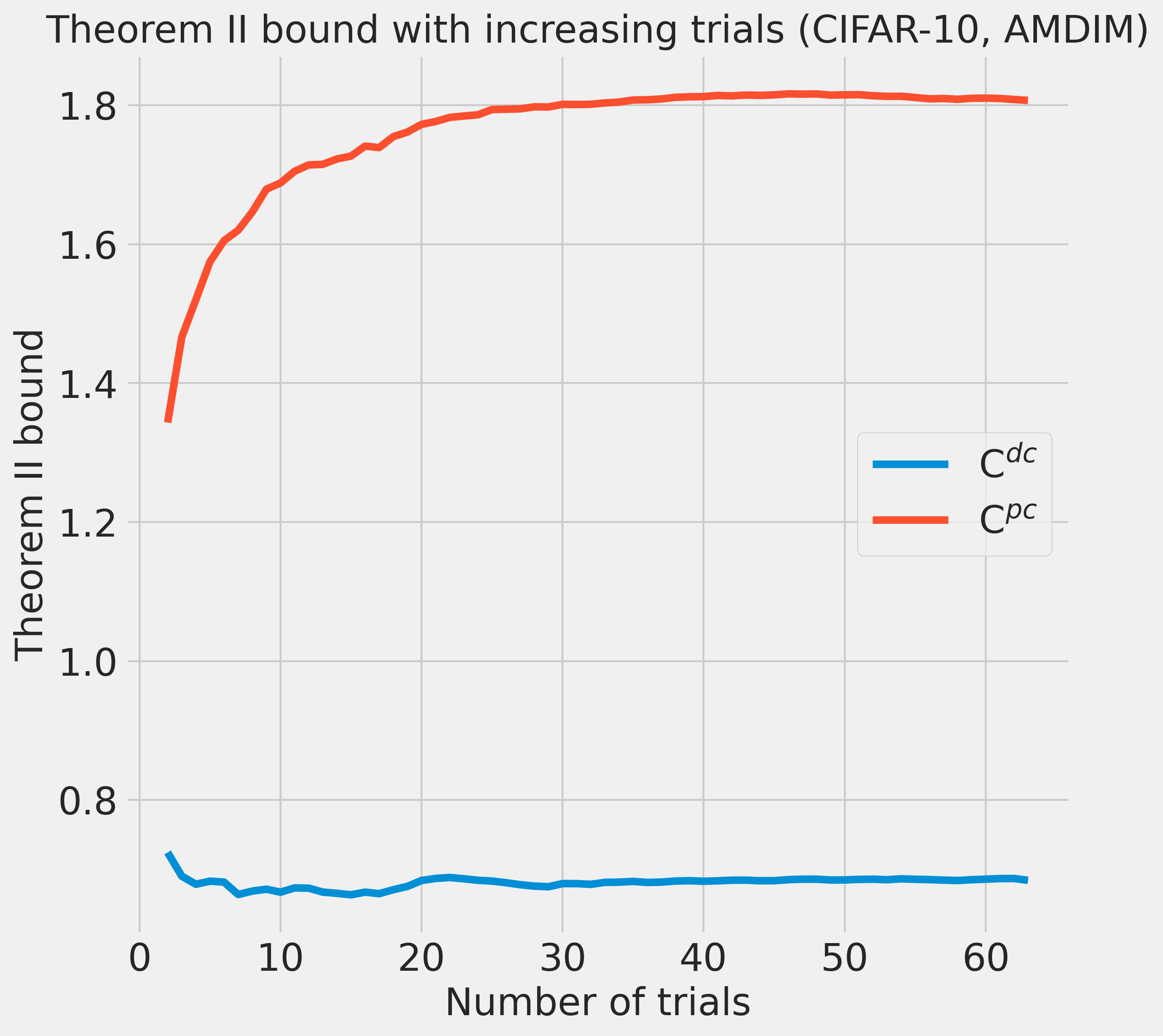}
    \caption{Theorem II bound with increasing trials. The bound based on $\cmi$ is lower than $\ccmi$ as expected, and converges within 20 trials.}
    \end{subfigure}\hspace{0.3cm}
    \begin{subfigure}{0.45\textwidth}
    \centering
    \includegraphics[width=\textwidth]{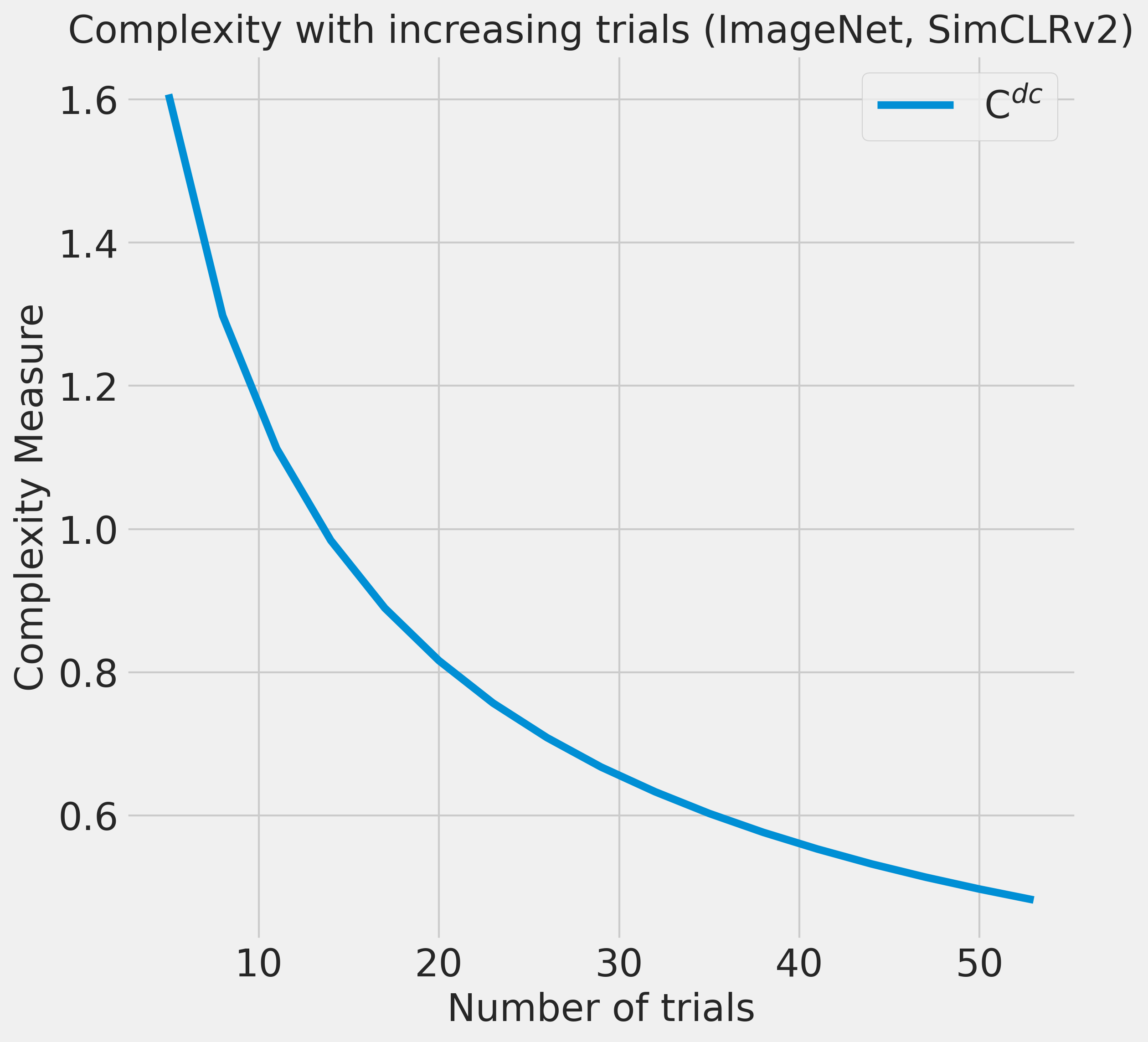}
    \caption{Theorem II bound with increasing trials. $\cmi$ is slow to converge due to the large dataset size (10 augmentations $\times$ 1.2 million training samples).}
    \end{subfigure}
    \caption{Convergence of Theorem II bounds for CIFAR-10 and ImageNet}
    \label{fig:cmi}
\end{figure}

%% file: app-robustness.tex
\section{Simple robustness bounds} \label{app:sec}

While robustness is not the focus of this work, we collect here two observations on the robustness of the least-square and minimum risk classifiers. These bounds are arguably folklore, but we state them here for completeness. 

\subsection{Robustness of least squares classifiers} \label{sec:LSE}

One can prove robustness for classes of algorithms under varying assumptions.
As a simple example, we record here a self-contained observation of how margin leads to robustness in least squares minimization.
This is a very simple but also pessimistic bound, and much better ones often hold.

\begin{lemma}
Let $x_1,\ldots,x_n \in \R^d$ and $y_1,\ldots,y_n \in [k]$, and consider a linear function $f:\R^d \rightarrow \R^k$ that minimizes the quantity 
$\sum_{i \in [n], j \in [k]} | f(x_i)_j - \characteristic_{y_i=j} |^2$, and suppose that for $p$ fraction of the $i$'s, the maximum over $j\in[k]$ of $f(x_i)$ is $\gamma$ larger than the second-largest value.

Then in expectation, if we let $\tilde{\vy}$ be the $\eta$-noisy version of $\vy$ and $\tilde{f}$ minimizes $\sum_{i \in [n], j \in [k]} | \tilde{f}(x_i)_j - \characteristic_{\tilde{y}_i=j} |^2$, we get that $\arg\max_j \tilde{f}(x_i) = y_i$ for at least $p-4\eta/\gamma^2$ fraction of the $i$'s.
\end{lemma}

\begin{proof}
We identify $\vy$ with its ``one hot'' encoding as a vector in $\R^{nk}$.
Let $V \subseteq \R^{nk}$ be the subspace of all vectors of the form $(g(x_1),\ldots,g(x_n))$ for linear $g:\R^d \rightarrow \R^k$.
If $f$ is the minimizer in the theorem statement, and $\vec{p}=(f(x_1),\ldots,f(x_n))$ then $\vec{p} = \Pi_V \vy$ where $\Pi_V$ is the orthogonal projection to the subspace $v$.
If $\tilde{f}$ is the minimizer for the noisy labels and $\tilde{\vec{p}} = (\tilde{f}(x_1),\ldots,\tilde{f}(x_n))$, then 
$\tilde{\vec{p}} = \Pi_V \tilde{\vy} = \Pi_V (\vy + \vec{e})$ where $\vec{e}$ is the noise vector $\tilde{\vy}-\vy$.

Hence $\| \vec{p} - \tilde{\vec{p}}\| = \| \Pi_V \vec{e} \| \leq \|\vec{e}\|$.
But in expectation $\|\vec{e}\|^2 \leq 2\eta n$ (since we flip a label with probability $\leq \eta$).
For every point $i$ for which the margin was at least $\gamma$ in $\vec{p}$, if $\tilde{\vec{p}}$'s prediction is different in $i$, then the contribution of the $i$-th block  to their square norm difference is at least $\gamma^2/2$ (by shifting the maximum coordinate by $-\gamma/2$ and the second largest one by $\gamma/2$).
Hence at most $4 \eta n / \gamma^2$ of these points could have different predictions in $\vec{p}$ and $\tilde{\vec{p}}$
\end{proof}

\subsection{Robustness of empirical risk minimizer} \label{sec:ERMrobust}

The (potentially inefficient) algorithm that minimizes the classification errors is always robust.

\begin{lemma}
Let $T(\vx,\vy) = \arg\min_{f \in\cF} \sum_{i=1}^n \characteristic_{f(x_i) \neq y_i}$. Then for every $\eta>0$,
$$\text{\textcolor{robustness}{Robustness gap}}(T) \leq 2\eta\;.$$
\end{lemma}

\begin{proof}
Let $\vx,\vy$ be any train set, and let  $\alpha = \min_{g \in\cF} \sum_{i=1}^n \characteristic_{g(x_i) \neq y_i}$ and $f$ be the minimizer of this quantity.
Let $\tilde{\vy}$ be the $\eta$-noisy version of $\vy$ and let $\tilde{\eta}$ be the fraction of $i$ on which $y_i \neq \tilde{y}_i$.
Then,  
\begin{equation}
\sum_{i=1}^n \characteristic_{f(x_i) \neq y_i} \leq \alpha +\tilde{\eta} \;. \label{eq:kghkerewrgfgfv}
\end{equation}
Hence if  $\tilde{f}$ is the minimizer of (\ref{eq:kghkerewrgfgfv}) then we know that $\tilde{f}(x_i) \neq \tilde{y}_i$ for at most  $\alpha+\tilde{\eta}$ fraction of the $i$'s, and so $\tilde{f}(x_i) \neq y_i$  for at most $\alpha + 2\tilde{\eta}$ fraction of the $i$'s.
Since the train accuracy of $T$ is $1-\alpha$ and in expectation of $\tilde{\eta}$ is $\eta$, we get that in expectation 
$$
\TrainAcc_{T}(\eta) \geq \TrainAcc_T - 2\eta
$$
\end{proof}

%% file: app-results-table.tex
\begin{table}[ht]
\caption{Summary of all the methods, architectures and the corresponding results (gaps and accuracies) on CIFAR-10, sorted by generalization gap. While Figure~\ref{fig:intro} already plots this data, here we also provide the test performance of the corresponding models.}
\label{tab:all-cifar}
\footnotesize
\begin{tabular}{lllrrrrlrr}
\toprule
 \makecell{Method} &         Backbone & \makecell{Data \\ Aug} &  \makecell{Generalization \\ Gap} &  \makecell{Robustness} &  \makecell{Mem-\\orization} &  \makecell{Rationality} &  \makecell{Theorem II \\ bound} &  \makecell{RRM \\ bound} &  \makecell{Test \\ Acc} \\
\midrule
 mocov2 &         resnet18 &              True &               -7.35 &        0.07 &          0.21 &         0.00 &              3.47 &       0.28 &             67.19 \\
 mocov2 &  wide\_resnet50\_2 &              True &               -6.37 &        0.18 &          1.03 &         0.00 &              7.63 &       1.21 &             70.99 \\
 mocov2 &        resnet101 &              True &               -6.01 &        0.15 &          0.71 &         0.00 &              6.38 &       0.86 &             68.58 \\
 mocov2 &         resnet50 &              True &               -5.38 &        0.19 &          0.84 &         0.00 &              6.99 &       1.03 &             69.68 \\
 simclr &         resnet50 &              True &               -2.89 &        0.30 &          0.55 &         0.00 &              6.63 &       0.85 &             91.96 \\
  amdim &        resnet101 &              True &               -0.91 &        0.64 &          3.70 &         0.00 &             25.99 &       4.34 &             63.56 \\
  amdim &         resnet18 &              True &                0.33 &        0.23 &          1.15 &         0.00 &              8.66 &       1.38 &             62.84 \\
 mocov2 &         resnet18 &             False &                1.43 &        0.15 &          1.24 &         0.03 &             14.14 &       1.43 &             67.60 \\
 simclr &         resnet18 &             False &                1.43 &        0.28 &          0.79 &         0.36 &             13.35 &       1.43 &             82.50 \\
  amdim &  wide\_resnet50\_2 &              True &                1.60 &        0.69 &          2.46 &         0.00 &             19.20 &       3.15 &             64.38 \\
 simclr &         resnet50 &             False &                1.97 &        0.22 &          0.78 &         0.97 &             15.75 &       1.97 &             92.00 \\
 simclr &         resnet50 &             False &                2.24 &        0.52 &          1.71 &         0.01 &             19.53 &       2.24 &             84.94 \\
 mocov2 &         resnet50 &             False &                2.72 &        0.30 &          2.96 &         0.00 &             24.18 &       3.26 &             70.09 \\
 mocov2 &        resnet101 &             False &                2.82 &        0.33 &          3.03 &         0.00 &             22.78 &       3.36 &             69.08 \\
 mocov2 &  wide\_resnet50\_2 &             False &                3.11 &        0.38 &          2.79 &         0.00 &             22.39 &       3.18 &             70.84 \\
  amdim &      resnet50\_bn &              True &                3.69 &        0.84 &          4.22 &         0.00 &             31.12 &       5.06 &             66.44 \\
  amdim &         resnet18 &             False &                4.34 &        0.42 &          4.58 &         0.00 &             33.47 &       5.00 &             62.28 \\
  amdim &    amdim\_encoder &              True &                4.43 &        0.68 &          0.36 &         3.39 &             10.32 &       4.43 &             87.33 \\
  amdim &    amdim\_encoder &             False &                6.68 &        2.08 &          5.69 &         0.00 &             70.52 &       7.77 &             87.38 \\
  amdim &        resnet101 &             False &               12.46 &        1.22 &         14.26 &         0.00 &            100.00 &      15.49 &             62.43 \\
  amdim &  wide\_resnet50\_2 &             False &               13.07 &        1.70 &         15.33 &         0.00 &            100.00 &      17.03 &             63.80 \\
  amdim &      resnet50\_bn &             False &               14.73 &        1.81 &         16.63 &         0.00 &            100.00 &      18.43 &             66.28 \\
\bottomrule
\end{tabular}
\end{table}

\begin{table}[ht]
\caption{Summary of all the methods, architectures their corresponding results (gaps and accuracies) on ImageNet, sorted by generalization gap. While Figure~\ref{fig:IN-MAIN} already  plots this data, here we also provide the test performance of the corresponding models.}
\label{tab:all-imagenet}
\footnotesize
\begin{tabular}{lllrrrrlrr}
\toprule
 \makecell{Method} &         Backbone & \makecell{Data \\ Aug} &  \makecell{Generalization \\ Gap} &  \makecell{Robustness} &  \makecell{Mem-\\orization} &  \makecell{Rationality} &  \makecell{Theorem II \\ bound} &  \makecell{RRM \\ bound} &  \makecell{Test \\ Acc} \\
\midrule
 simclrv2 &   r50\_1x\_sk0 &               True &               -2.34 &        0.26 &          0.68 &         0.00 &               46.93 &       0.94 &             70.96 \\
 simclrv2 &  r101\_2x\_sk0 &               True &                0.63 &        0.10 &          0.80 &         0.00 &               47.90 &       0.91 &             77.24 \\
 simclrv2 &  r152\_2x\_sk0 &               True &                1.00 &        0.13 &          0.77 &         0.10 &               NA &       1.00 &             77.65 \\
     moco &    ResNet-50 &               True &                1.32 &        0.57 &          0.93 &         0.00 &               NA &       1.49 &             70.15 \\
  InfoMin &    ResNet-50 &               True &                4.88 &        0.81 &          1.01 &         3.06 &               NA &       4.88 &             72.29 \\
     PiRL &    ResNet-50 &               True &                6.23 &        0.29 &          0.99 &         4.95 &               NA &       6.23 &             60.56 \\
   InsDis &    ResNet-50 &               True &                6.85 &        0.25 &          1.13 &         5.46 &               NA &       6.85 &             58.30 \\
 simclrv2 &  r101\_1x\_sk1 &              False &                8.23 &        0.71 &          4.66 &         2.86 &               NA &       8.23 &             76.07 \\
  InfoMin &    ResNet-50 &              False &               10.21 &        2.34 &          8.96 &         0.00 &               NA &      11.31 &             70.31 \\
 simclrv2 &  r152\_1x\_sk0 &              False &               10.32 &        1.12 &          6.93 &         2.26 &               NA &      10.32 &             74.17 \\
 simclrv2 &  r101\_1x\_sk0 &              False &               10.53 &        1.11 &          6.99 &         2.42 &               NA &      10.53 &             73.04 \\
 simclrv2 &   r50\_1x\_sk0 &              False &               10.62 &        0.99 &          7.31 &         2.31 &               NA &      10.62 &             70.69 \\
     moco &    ResNet-50 &              False &               10.72 &        1.82 &          7.86 &         1.04 &               NA &      10.72 &             68.39 \\
 simclrv2 &  r152\_2x\_sk0 &              False &               10.92 &        0.75 &          7.45 &         2.72 &               NA &      10.92 &             77.25 \\
 simclrv2 &  r101\_2x\_sk0 &              False &               11.02 &        0.74 &          7.51 &         2.78 &               NA &      11.02 &             76.72 \\
   simclr &  ResNet50\_1x &              False &               11.07 &        1.22 &          7.73 &         2.13 &               NA &      11.07 &             68.73 \\
 simclrv2 &    ResNet-50 &              False &               11.16 &        0.64 &          7.67 &         2.85 &               NA &      11.16 &             74.99 \\
     PiRL &    ResNet-50 &              False &               11.43 &        1.49 &          8.26 &         1.68 &               NA &      11.43 &             59.11 \\
   InsDis &    ResNet-50 &              False &               12.02 &        1.40 &          8.52 &         2.10 &               NA &      12.02 &             56.67 \\
    amdim &    ResNet-50 &              False &               13.62 &        0.90 &          9.72 &         3.01 &               NA &      13.62 &             67.69 \\
      CMC &    ResNet-50 &              False &               14.73 &        2.30 &         12.30 &         0.13 &               NA &      14.73 &             54.60 \\
 bigbigan &    ResNet-50 &              False &               29.60 &        3.13 &         25.19 &         1.27 &               NA &      29.60 &             50.24 \\
\bottomrule
\end{tabular}
\end{table}

%% file: hyperparam-table.tex
\normalsize

\footnotesize
\begin{table}[hb]
\caption{\label{tab:hyperparam-table-cifar} Summary of training methods with their hyper-parameters for CIFAR-10}
\centering
\begin{tabular}{@{}cccccccc@{}} 
    \toprule
    \makecell{Self- \\ supervised \\ method}  & \makecell{Backbone \\ Architectures} & \makecell{Self-supervised \\ Training} & Evaluation & \makecell{Simple \\ Phase \\ Optimization}\\ 
    \midrule

    \multirow{5}{*}{AMDIM}  & AMDIM Encoder  & \multirow{5}{*}{\makecell{PLB \\ Default \\ parameters}}  & \multirow{5}{*}{Linear}    & \multirow{5}{*}{\makecell{Adam \\ $\beta_1 = 0.8$ $\beta_2 = 0.999$ \\ Constant LR = 2e-4 \\ Batchsize = 500 \\ Weight decay = 1e-6}}\\ 
    
                        & ResNet-18 &   &     &  &  & \\ 
                        
                        & ResNet-50 &   &     &  &  &  \\ 
                        
                        & WideResNet-50 &   &     &  &  &  \\     
                        
                        & ResNet 101 &   &     &  &  \\
           \\             
    \hline
    \\            
    \multirow{5}{*}{MoCoV2}  & ResNet-18 & \multirow{5}{*}{\makecell{PLB \\ Default \\ parameters}}  & \multirow{5}{*}{Linear}    & \multirow{5}{*}{\makecell{Adam \\ $\beta_1 = 0.8$ $\beta_2 = 0.999$ \\ Constant LR = 2e-4 \\ Batchsize = 500 \\ Weight decay = 1e-6}} \\

                        & ResNet-50 &   &     &  &  &  \\ 
                        
                        & WideResNet-50 &   &     &  &  &  \\     
                        
                        & ResNet 101 &   &     &  &  \\

     \\
    \hline
    \\
    
    \multirow{4}{*}{SimCLR}  & ResNet-18 & \multirow{2}{*}{\makecell{Batchsize = 128 \\ Epochs 200}}  & \multirow{4}{*}{Linear}    & \multirow{4}{*}{\makecell{SGD \\ Momentum = 0.9 \\ Constant LR = 0.1 \\ Weight decay 1e-6}} \\ 
    
                        & ResNet-50  &   &     &  &  &  \\ 
                        
                        & &   &     &  &  & \\
                        
                        & ResNet-50 & \makecell{Batchsize = 512 \\ Epochs 600}  &     &  &  &  \\

    \bottomrule
\end{tabular} \\
\end{table}
\normalsize

\footnotesize
\begin{table}[ht]
\caption{\label{tab:hyperparam-table-imagenet} Summary of training methods with their hyper-parameters for ImageNet}

\begin{tabular}{@{}cccccccc@{}} 
    \toprule
    \makecell{Self-supervised \\ method}  & \makecell{Backbone \\ Architecture} & \makecell{Pre-trained \\ Model} & Evaluation & \makecell{Optimization} & \makecell{Weight \\ Decay} & Epochs \\ 
    \midrule
    
    \makecell{Instance \\ Discrimination}  & ResNet-50 & PyContrast  & Linear    & \makecell{SGD \\ Momentum = 0.9 \\ Initial LR = 30 \\ LR drop at \{30\} \\ by factor 0.2} & 0 & 40 \\ 
    
    \\
    \hline
    \\
    
    \makecell{MoCo}  & ResNet-50 & Official  & Linear    & \makecell{SGD \\ Momentum = 0.9 \\ Initial LR = 30 \\ LR drop at \{30\} \\ by factor 0.2} & 0 & 40 \\ 
    
    \\
    \hline
    \\
    
    \makecell{PiRL}  & ResNet-50 & PyContrast  & Linear    & \makecell{SGD \\ Momentum = 0.9 \\ Initial LR = 30 \\ LR drop at \{30\} \\ by factor 0.2} & 0 & 40 \\ 
    
    \\
    \hline
    \\
    
    \makecell{CMC}  & ResNet-50 & PyContrast  & Linear    & \makecell{SGD \\ Momentum = 0.9 \\ Initial LR = 30 \\ LR drop at \{30\} \\ by factor 0.2} & 0 & 40 \\ 
    
    \\
    \hline
    \\
    
    \makecell{AMDIM}  & AMDIM Encoder & Official  & Linear    & \makecell{SGD \\ Momentum = 0.9 \\ Initial LR = 30 \\ LR drop at \{15, 25\} \\ by factor 0.2} & 1e-3 & 40 \\   
    
    \\
    \hline
    \\
    
    \makecell{BigBiGAN}  & ResNet-50 & Official  & Linear    & \makecell{SGD \\ Momentum = 0.9 \\ Initial LR = 30 \\ LR drop at \{15, 25\} \\ by factor 0.2} & 1e-5 & 40 \\       
    
    \\
    \hline
    \\
    
    \multirow{2}{*}{SimCLRv1}  & ResNet-50 1x & \multirow{2}{*}{Official}  & \multirow{2}{*}{Linear}    & \multirow{2}{*}{\makecell{SGD \\ Momentum = 0.9 \\ Constant LR = 0.1}} & \multirow{2}{*}{1e-6} & \multirow{2}{*}{40} \\ 
    
                        & ResNet-50 4x &   &     &  &  &  \\ 
    
    \\
    \hline
    \\
    
    \multirow{4}{*}{SimCLRv2}  & ResNet-50 1x SK0 & \multirow{4}{*}{Official}  & \multirow{4}{*}{Linear}    & \multirow{4}{*}{\makecell{SGD \\ Momentum = 0.9 \\ Constant LR = 0.1}} & \multirow{4}{*}{1e-6} & \multirow{4}{*}{40} \\ 
    
                        & ResNet-101 2x SK0 &   &     &  &  &  \\ 
                        
                        & ResNet-152 2x SK0 &   &     &  &  &  \\ 
                        
                        & ResNet-152 3x SK0 &   &     &  &  &  \\ 
    
    \bottomrule
\end{tabular} \\
\end{table}